\documentclass[acmsmall]{acmart}

\usepackage{lineno}   % dzl
\usepackage{booktabs} % For formal tables
\usepackage{hyperref}   %gws 20170930
\usepackage{url}  %gws 20170930

\usepackage{multirow}  %gws 20170930
\usepackage{color}  %gws 20170930
\newtheorem{strategy}{Strategy}    % gws
\newtheorem{upper bound}{Upper bound}

\usepackage{diagbox}
\usepackage{bm}
\usepackage{amsmath}
\usepackage{arydshln} % gws, 20180202
\usepackage{color} % gws, 20180120

\usepackage[linesnumbered,ruled]{algorithm2e} % For algorithms

\SetAlFnt{\small}
\SetAlCapFnt{\small}
\SetAlCapNameFnt{\small}
\SetAlCapHSkip{0pt}
\IncMargin{-\parindent}

\acmJournal{JACM} %TKDD   % JACM

\setcopyright{acmcopyright}
\setcopyright{acmlicensed}
\acmDOI{0000001.0000001}

\begin{document}
	
%\linenumbers
% Title portion
\title{Totally-ordered Sequential Rules for Utility Maximization}

\author{Chunkai Zhang}
\affiliation{
	\institution{Harbin Institute of Technology (Shenzhen)}
	\city{Shenzhen}
	\country{China}
}
\email{ckzhang@hit.edu.cn}

\author{Maohua Lyu}
\affiliation{
 	\institution{Harbin Institute of Technology (Shenzhen)}
 	\city{Shenzhen}
 	\country{China}
 }
\email{21s151083@stu.hit.edu.cn}

\author{Wensheng Gan}
\authornote{This is the corresponding author}
\affiliation{
	\institution{Jinan University}
	\city{Guangzhou}
	\country{China}
}
\email{wsgan001@gmail.com}

\author{Philip S. Yu}
\affiliation{
	\institution{University of Illinois at Chicago}
	\city{Chicago}
	\country{USA}
}
\email{psyu@uic.edu}

\begin{abstract}

High utility sequential pattern mining (HUSPM) is a significant and valuable activity in knowledge discovery and data analytics with many real-world applications. In some cases, HUSPM can not provide an excellent measure to predict what will happen. High utility sequential rule mining (HUSRM) discovers high utility and high confidence sequential rules, allowing it to solve the problem in HUSPM. All existing HUSRM algorithms aim to find high-utility partially-ordered sequential rules (HUSRs), which are not consistent with reality and may generate fake HUSRs. Therefore, in this paper, we formulate the problem of high utility totally-ordered sequential rule mining and propose two novel algorithms, called TotalSR and TotalSR$^+$, which aim to identify all high utility totally-ordered sequential rules (HTSRs). TotalSR creates a utility table that can efficiently calculate antecedent support and a utility prefix sum list that can compute the remaining utility in $O(1)$ time for a sequence. We also introduce a left-first expansion strategy that can utilize the anti-monotonic property to use a confidence pruning strategy. TotalSR can also drastically reduce the search space with the help of utility upper bounds pruning strategies, avoiding much more meaningless computation. In addition, TotalSR$^+$ uses an auxiliary antecedent record table to more efficiently discover HTSRs. Finally, there are numerous experimental results on both real and synthetic datasets demonstrating that TotalSR is significantly more efficient than algorithms with fewer pruning strategies, and TotalSR$^+$ is significantly more efficient than TotalSR in terms of running time and scalability. 
	
\end{abstract}



%H.2.8 [Database Applications]: Data mining



\maketitle

\renewcommand{\shortauthors}{C. Zhang et al.}

\section{Introduction} \label{sec1}

In the age of rapid data generation, how we effectively and efficiently organize and analyze the underlying relationships and knowledge of data to lead a more productive life is a useful, meaningful, and challenging task. Knowledge discovery, e.g., pattern mining, which is one of the subareas of data mining and can be applied to many applications in the real world, is just such a method that helps people discover and analyze hidden relationships and knowledge in data.

Frequent pattern mining (FPM) \cite{agrawal1994fast} aims to find all frequent items that appear together, and these items are interesting and useful knowledge, or so-called frequent patterns with the occurrence times no less than the minimum support (\textit{minsup}) value defined by the user in the transaction database. Moreover, FPM can play a significant role in market-decision \cite{maske2018survey}, biological medicine diagnosis \cite{nayak2019heart}, weblog mining \cite{ahmed2010mining}, and so on. In general, FPM can only find patterns that occur at the same time. Therefore, FPM makes it difficult to find patterns for the more complicated situations where items occur in chronological order. Consequently, sequential pattern mining (SPM) \cite{agrawal1995mining,2004Mining, fournier2017surveys, gan2019survey} was proposed to address the need for mining sequential patterns in sequential databases. We can view SPM as a generalization of FPM since SPM can process more complex data than FPM. Correspondingly, there are more applications in real life that SPM can be involved in due to the fact that there is much more information that can be encoded as sequential symbols. As reviewed in \cite{fournier2017surveys}, the real world applications that SPM can be involved in not only include the sequential symbol data such as text analysis \cite{pokou2016authorship}, market basket analysis \cite{srikant1996mining}, bio-informatics \cite{wang2007frequent}, web-page click-stream analysis \cite{fournier2012using}, and e-learn \cite{ziebarth2015resource}, but also include the time series data such as stock data when there is a discretization process before mining \cite{lin2007experiencing}.

However, FPM and SPM regard the frequency or occurrence times in a database as constant, and all items in the database have the same weight, which is not committed to reality. In fact, for instance, in the selling data of a supermarket, we should not only consider the quantity of the commodities (items) but also the price of each commodity, for both the commodity's quantity and price matter, which can bring high profit to people. For example, in the electrical appliance store, the profit of $<$\textit{television}, \textit{refrigerator}$>$ is undoubtedly greater than the profit of $<$\textit{socket}, \textit{wire}$>$. In FPM and SPM, because the demand for $<$\textit{socket}, \textit{wire}$>$ is certainly greater than the demand for $<$\textit{television}, \textit{refrigerator}$>$, the pattern $<$\textit{socket}, \textit{wire}$>$ with the larger support count will be considered more valuable than $<$\textit{television}, \textit{refrigerator}$>$. Moreover, if we set a too high minimum support threshold, the pattern $<$\textit{television}, \textit{refrigerator}$>$ will be missed in the mining process, which is not the result that we want. To solve the limitation mentioned above, utility-oriented pattern mining has been proposed. High-utility pattern mining (HUPM) was proposed by Ahmed \textit{et al.} \cite{ahmed2010mining}, who introduced an economic concept, utility. HUPM aims to find all high-utility patterns, which means the utility of a pattern must satisfy the minimum utility threshold (\textit{minutil}). In other words, HUPM takes both the quantity (internal utility) and unit utility (external utility) of each item into account. For example, the internal utility of the item $<$\textit{television}$>$ is the purchased quantity, and the external utility of that item is its price. Furthermore, the utility of one specific item equals its internal utility times its external utility. Unfortunately, utility-oriented pattern mining is much more challenging than FPM and SPM. Since anti-monotonicity is useless in utility-oriented mining, this property in frequent mining can easily identify candidate patterns. In the previous example, the price of the pattern $<$\textit{television}, \textit{refrigerator}$>$ is higher than the price of the pattern $<$\textit{socket}, \textit{wire}$>$. However, the quantity of the pattern $<$\textit{television}, \textit{refrigerator}$>$ is lower than that of the pattern $<$\textit{socket}, \textit{wire}$>$. As a result, comparing the exact utility of the pattern $<$\textit{television}, \textit{refrigerator}$>$ and $<$\textit{socket}, \textit{wire}$>$ is difficult. Like FPM, it is hard to handle the items that form in chronological order for HUPM. The idea of high-utility sequential pattern mining (HUSPM) \cite{yin2012uspan} was consequently put forth.

Plentiful algorithms for SPM and HUSPM have been proposed to improve the efficiency regarding time and memory consumption and to apply to some distinct scenarios in the real world. Nevertheless, all these algorithms face trouble in that all of them just extract the set of items, which can not provide an excellent measure to predict what will happen and what the probability of the items is if they appear after a specific sequence. Therefore, Fournier-Viger \textit{et al.} \cite{fournier2015mining} proposed sequential rule mining (SRM) that uses the concept of confidence so that rules found by SRM can provide confidence. In other words, the goal of SRM is to identify all sequential rules (SRs) that satisfy the \textit{minsup} and minimum confidence threshold (\textit{minconf}). Generally speaking, SR is represented as $<$$X$$>$ $\rightarrow$ $<$$Y$$>$, where $X$ represents the antecedent of SR and $Y$ stands for its consequent. For example, a SR $<$\{\textit{television}, \textit{refrigerator}\}$>$ $\rightarrow$ $<$\textit{air-conditioning}$>$ with a confidence equals 0.6 means that there is a probability of 60\% that customer will buy air-conditioning after purchasing television and refrigerator. SRM, like SPM, does not account for the utility because it considers all items of equal importance. Correspondingly, high-utility sequential rule mining (HUSRM) \cite{zida2015efficient}, which aims to identify high-utility sequential rules (HUSRs), was proposed. However, HUSRM discovers the partially-ordered HUSR, which only requires that items in the antecedent of a HUSR occur earlier than the consequent. HUSRM does not consider the inner ordering of the antecedent or consequent of a HUSR. Unfortunately, in reality, the sequential relationship between each item does matter. Thinking about two sequences $<$\textit{heart attack}, \textit{emergency  measure}, \textit{go to hospital}, \textit{survival}$>$ and $<$\textit{heart attack}, \textit{go to hospital}, \textit{emergency  measure}, \textit{death}$>$, HUSRM may output a rule: $<$\{\textit{heart attack}, \textit{emergency  measure}, \textit{go to hospital}\}$>$ $\rightarrow$ $<$\textit{survival}$>$, but the second sequence can also form the left part of the rule, which produces a totally different result from the original sequences, that is \textit{death} to \textit{survival}. Moreover, in news recommendations, for instance, breaking news will occur in a particular order. If we omit the order in which events appear, it may cause some trouble. Thus, if we only simplify the sequence into two parts such that in each part we do not care about the inner order, we may get an ambiguous HUSR. In addition, the partially-ordered sequential rule mining may generate fake HUSR. For example, there are two totally-ordered rules $<$\textit{a}, \textit{b}$>$ $\rightarrow$ $<$\textit{c}, \textit{d}$>$ and $<$\textit{a}, \textit{b}$>$ $\rightarrow$ $<$\textit{d}, \textit{c}$>$ with the utility is equivalent to half \textit{minutil}, respectively. Note that we assume all two totally-ordered rules are high confidence. Thus, we can obtain a partially-ordered rule $<$\textit{a}, \textit{b}$>$ $\rightarrow$ $<$\textit{c}, \textit{d}$>$ that is high-utility. However, it is not consistent with the actual result. The more complex the relationship that the sequence can form, the more severe the phenomenon of fake SR can become. Therefore, it is necessary to formulate an algorithm for discovering high-utility totally-ordered sequential rules.

To the best of our knowledge, there is no work that discovers high-utility totally-ordered sequential rules. For the sake of the limitations of partially-ordered sequential rule mining, in this paper, we formulate the problem of totally-ordered sequential rule mining (ToSRM) and propose an algorithm called TotalSR and its optimized version TotalSR$^+$, which aims to find all high-utility totally-ordered sequential rules (HTSRs) in a given sequential database. However, in ToSRM, the support of the antecedent of a rule is naturally difficult to measure, since we usually only know the support of this rule. In TotalSR we designed a special data structure, the utility table, which records the sequences that the antecedent of totally-ordered sequential rules appears in to efficiently calculate the support of the antecedent and reduce the memory consumption. Besides, in TotalSR$^+$ we redesign the utility table and propose an auxiliary antecedent record table to reduce execution time compared to TotalSR. Moreover, we use a left-first expansion strategy to utilize the confidence pruning strategy, which can make great use of the anti-monotonic property to avoid the invalid expansion of low confidence HTSR. Inspired by the remaining utility \cite{wang2016efficiently, gan2020proum, gan2020fast}, in this paper, we proposed two novel utility-based upper bounds, named the left expansion reduced sequence prefix extension utility (\textit{LERSPEU}) and the right expansion reduced sequence prefix extension utility (\textit{RERSPEU}). In addition, to fast compute the remaining utility for a given sequence, we designed a data structure called utility prefix sum list (\textit{UPSL}), which can calculate the remaining utility value of the given sequence in \textit{O}(1) time. Based on the strategies and the data structures mentioned above, our algorithm can efficiently identify all HTSRs. The main contributions of this work can be outlined as follows:

\begin{itemize}
	\item We formulated the problem of totally-ordered sequential rule mining and proposed two algorithms, TotalSR and its optimized version TotalSR$^+$, which can find the complete set of HTSRs in a given sequential database.
	
	\item A left-first expansion strategy was introduced, which can utilize the anti-monotonic property of confidence to prune the search space. Besides, with the help of the utility table and \textit{UPSL}, we can avoid scanning the database repeatedly and quickly compute the remaining utility of each sequence and calculate the upper bounds. Therefore, TotalSR can avoid unnecessary expansions and tremendously reduce the search space.
	
	\item In order to extract HTSRs more effectively and efficiently, we proposed an optimized algorithm called TotalSR$^+$, which redesigns the utility table and introduces an auxiliary antecedent record table. TotalSR$^+$ can significantly reduce execution time and is thus more efficient than TotalSR.
	
	\item Experiments on both real and synthetic datasets show that TotalSR with all optimizations is much more efficient compared to those algorithms that use only a few optimizations. Furthermore, experimental findings demonstrate that TotalSR$^+$ is far more effective than TotalSR.
\end{itemize}

The rest of this paper is organized as follows. In Section \ref{sec:relatedwork}, we briefly review the related work on HUSPM, SRM, and HUSRM. The basic definitions and the formal high-utility totally-ordered sequential rule mining problem are introduced in Section \ref{sec:prelimis}. The proposed algorithm, TotalSR, and its optimized version, TotalSR$^+$, as well as the corresponding pruning strategies are provided in Section \ref{sec:techs}. In Section \ref{sec:exps}, we show and discuss the experimental results and evaluation of both real and synthetic datasets. Finally, the conclusions of this paper and future work are discussed in Section \ref{sec:concs}.
\section{Related Work}  \label{sec:relatedwork}

There is a lot of work on high-utility sequential pattern mining (HUSPM) and sequential rule mining (SRM), but there is little work on high-utility sequential rule mining (HUSRM). In this section, we separately review the prior literature on HUSPM, HUSRM, and HUSRM.

\subsection{High-utility sequential pattern mining}

To make the found sequential patterns meet the different levels of attention of users, utility-oriented sequential pattern mining, which is to mine the patterns that satisfy a minimum utility threshold defined by users, has been widely developed in recent years. However, since utility is neither monotonic nor anti-monotonic, the utility-based pattern mining method does not possess the Apriori property, which makes it difficult to discover patterns in the utility-oriented framework compared to the frequency-based framework. Ahmed \textit{et al.} \cite{ahmed2010mining} developed the sequence weighted utilization (\textit{SWU}) to prune the search space. They also designed two tree structures, UWAS-tree and IUWAS-tree to discover high-utility sequential patterns (HUSPs) in web log sequences. With the help of \textit{SWU}, which has a downward closure property based on the utility-based upper bound \textit{SWU}, the search space can be pruned like the Apriori property. After that UitilityLevel and UtilitySpan \cite{ahmed2010novel} were proposed based on the \textit{SWU}, in which they first generated all candidate patterns and then selected the HUSPs. Therefore, they are time-consuming and memory-costing algorithms. UMSP \cite{shie2011mining} applied HUSPM to analyze mobile sequences. However, all these algorithms \cite{ahmed2010mining, ahmed2010novel, shie2011mining} assume that each itemset in a sequence only contains one item. Therefore, the applicability of these algorithms is hard to expand. After that, USpan \cite{yin2012uspan} introduced a data structure called utility-matrix, which can discover HUSPs from the sequences consisting of multiple items in each itemset to help extract HUSPs. In addition, USpan utilized the upper bound \textit{SWU} to efficiently find HUSPs. However, there is a big gap between the \textit{SWU} and the exact utility of a HUSP, which means the \textit{SWU} upper bound will produce too many unpromising candidates. To solve the problem of \textit{SWU}, HUS-Span \cite{wang2016efficiently}, which can discover all HUSPs by generating fewer candidates, introduced two other utility-based upper bounds, prefix extension utility (\textit{PEU}) and reduced sequence utility (\textit{RSU}). However, the efficiency of Hus-span is still not good enough. To more efficiently discover all HUSPs, Gan \textit{et al.} \cite{gan2020proum} proposed a novel algorithm ProUM that introduces a projection-based strategy and a new data structure called the utility array. ProUM can extend a pattern faster and take up less memory based on the projection-based strategy. HUSP-ULL \cite{gan2020fast} introduced a data structure called UL-list, which can quickly create the projected database according to the prefix sequence. Besides, HUSP-ULL also proposed the irrelevant item pruning strategy that can remove the unpromising items in the remaining sequences to reduce the remaining utility, i.e., to generate a tighter utility-based upper bound.

In addition to improving the efficiency of the HUSPM algorithms, there are also many algorithms that apply HUSPM to some specific scenarios. OSUMS \cite{zhang2021shelf} integrated the concept of on-shelf availability into utility mining for discovering high-utility sequences from multiple sequences. To get the fixed numbers of HUSPs and avoid setting the minimum utility threshold, which is difficult to determine for different datasets, TKUS \cite{zhang2021tkus} was the algorithm that only mined top-$k$ numbers of HUSPs. CSPM \cite{zhang2021utility} was the algorithm that required the itemset in the pattern to be contiguous, which means that the itemsets in the pattern found in CSPM occur consecutively. In order to acquire HUSPs consisting of some desired items, Zhang \textit{et al.} \cite{zhang2022tusq} proposed an algorithm called TUSQ for targeted utility mining. By integrating the fuzzy theory, PGFUM \cite{gan2021explainable} was proposed to enhance the explainability of the mined HUSPs. 

\subsection{Sequential rule mining}  \label{section:FRM} 

In order to be able to predict the probability of the occurrence of the next sequence well while mining the pattern, sequential rule mining (SRM) was proposed as a complement to sequential pattern mining (SPM) \cite{fournier2017surveys, gan2019survey,wu2021ntp,wu2022top}. Differing from SPM, a sequential rule (SR) counts the additional condition of confidence, which means a SR should not be less than both the conditions of \textit{minsup} and \textit{minconf}. A SR is defined as $<$$X$$>$ $\rightarrow$ $<$$Y$$>$ and $X$ $\cap$ $Y$ = $\varnothing$, where $X$ and $Y$ are subsequences from the same sequence. In general, according to the SR forming as partially-ordered or totally-ordered, there are two types of SRM: partially-ordered SRM and totally-ordered SRM. The first type of rule indicates that both antecedent and consequent in a SR are unordered sets of items \cite{fournier2011rulegrowth, fournier2014erminer, fournier2015mining}. But the items that appear in the consequent must be after the items in the antecedent, which means a partially-ordered SR consists of only two itemsets formed by disorganizing the original itemsets in the given sequence. The second type of rule states that both antecedent and consequent are sequential patterns \cite{lo2009non, pham2014efficient}. In other words, both the antecedent and consequent follow the original ordering in the given sequence. There are lots of algorithms for sequential rule mining. Sequential rule mining was first proposed by Zaki \textit{et al.} \cite{zaki2001spade}. They first mined all sequential patterns (SPs) and then generated SRs based on SPs, which is inefficient. Since it mines all SP as the first step, the SR they obtained belonged to the totally-ordered SR. To improve the efficiency of the SRM, CMRules \cite{fournier2012cmrules} introduced the left and right expansion strategies. After that, RuleGrowth \cite{fournier2011rulegrowth} and TRuleGrowth \cite{fournier2015mining} that discovered partially-ordered sequential rules were proposed, in which they used the left and right expansions to help the partially-ordered SR growth just like PrefixSpan \cite{han2001prefixspan}. The authors only extracted partially-ordered sequential rules because they explained too many similar rules in the results of SRs. Therefore, the partially-ordered SR can simultaneously represent the correct result and reduce the number of rules. Since the partially-ordered SRM does not care about the ordering in the inner antecedent or consequent, it can simplify the mining process. Therefore, ERMiner \cite{fournier2014erminer}, which makes great use of the property of partially-ordered, proposed a data structure, Sparse Count Matrix, to prune some invalid rules generation and improve efficiency. Lo \textit{et al.} \cite{lo2009non} proposed a non-redundant sequential rule mining algorithm that discovers the non-redundant SR, in which each rule cannot be a sub-rule of the other rule. Pham \textit{et al.} \cite{pham2014efficient} enhanced the efficiency of non-redundant SRM based on the idea of prefix-tree. Gan \textit{et al.} \cite{gan2022towards} proposed an SRM algorithm that can discover target SRs.

\subsection{High-utility sequential rule mining}

Although SRM \cite{fournier2011rulegrowth, fournier2012cmrules, fournier2014erminer, fournier2015mining} can provide the probability of the next sequence to users, it just finds the rules that satisfy the frequency requirement, which may omit some valuable but infrequent rules. Therefore, Zida \textit{et al.} \cite{zida2015efficient} introduced the utility concept into SRM and proposed a utility-oriented sequential rule mining algorithm called HUSRM. Similarly to \cite{fournier2011rulegrowth, fournier2012cmrules, fournier2014erminer, fournier2015mining}, HUSRM also used the partially-ordered sequential rule mining method and introduced a data structure called the utility table to maintain the essential information about candidate rules for expansion. Besides, HUSRM designed a bit map to calculate the support value of the antecedent and made some optimizations to improve the efficiency. Afterward, Huang \textit{et al.} \cite{huang2021us} proposed an algorithm called US-Rule to enhance the efficiency of the algorithm. Inspired by \textit{PEU} and \textit{RSU} from HUSPM in US-Rule, to remove useless rules, they proposed four utility-based upper bounds: left and right expansion estimated utility (\textit{LEEU} and \textit{REEU}), left and right expansion reduced sequence utility (\textit{LERSU} and \textit{RERSU}). There are also some extensions to high-utility sequential rule mining. DUOS \cite{gan2021anomaly} extracted unusual high-utility sequential rules, i.e., to detect the anomaly in sequential rules. DOUS is the first work that links anomaly detection and high-utility sequential rule mining. Zhang \textit{et al.} \cite{zhang2020hunsr} addressed HUSRM with negative sequences and proposed the e-HUNSR algorithm. HAUS-rules \cite{segura2022mining} introduced the high average-utility concept into sequential rule mining and found all high average-utility sequential rules in the gene sequences. However, all of these approaches address partially-ordered high-utility sequential rules.

\section{Definitions and Problem Description}
\label{sec:prelimis}

In this section, we first introduce some significant definitions and notations used in this paper. Then, the problem of high-utility totally-ordered sequential rule mining is formulated.

\subsection{Preliminaries}

\begin{definition}[Sequence database]
	\rm Let $I$ $=$ \{$i_1$, $i_2$, $\cdots$, $i_q$\} be a set of distinct items. An itemset (also called element) $I_k$ is a nonempty subset of $I$, that is $I_k$ $\subseteq$ $I$. Note that each item in an itemset is unordered. Without loss of generality, we assume that every item in the same itemset follows the \textit{lexicographical order} $\succ_{lex}$, that means $a$ $\textless$ $b$ $\textless$ $\cdots$ $\textless$ $z$. Besides, we will omit the brackets if an itemset only contains one item. A sequence $s$ $=$ $<$$e_1$, $e_2$, $\cdots$, $e_m$$>$, where $e_i$ $\subseteq$ $I$ ($1$ $\le$ $i$ $\le$ $m$), is consisted of a set of ordered itemsets. A sequential database $\mathcal{D}$ is consisted of a list of sequences, $\mathcal{D}$ $=$ $<$$s_1$, $s_2$, $\cdots$, $s_p$$>$, where $s_i$ $(1$ $\le$ $i$ $\le$ $p)$ is a sequence and each sequence has a unique identifier (\textit{SID}). Each distinct item $i$ $\in$ $I$ has a positive number that represents its external utility and designated as $iu(i)$. In addition, each item $i$ in a sequence $s_k$ has an internal utility that is represented by a positive value and is designated as $q$($i$, $s_k$). Similar to HUSRM \cite{zida2015efficient} and US-Rule \cite{huang2021us}, in this paper, we also assume that each sequence can only contain the same item at most once.
\end{definition}

\begin{definition}[Position and index of item]
	\rm Given a sequence $s_k$ $=$ $<$$e_1$, $e_2$, $\cdots$, $e_m$$>$, the position of an item $i$ is defined as the index of the itemset that item $i$ occurs, and the index of the item $i$ is the item index itself.
\end{definition}

\begin{table}[h]
	\centering
	\caption{Sequence database}
	\label{table1}
	\begin{tabular}{|c|c|}  
		\hline 
		\textbf{SID} & \textbf{Sequence} \\
		\hline  
		\(s_{1}\) & $<$\{($a$, 2) ($b$, 1)\} ($c$, 2) \{($d$, 4) ($f$, 2)\}$>$ \\ 
		\hline
		\(s_{2}\) & $<$\{($a$, 1) ($b$, 3)\} \{($e$, 1) ($f$, 1)\} ($d$, 2) ($c$, 1) ($h$, 1)$>$ \\  
		\hline  
		\(s_{3}\) & $<$\{($e$, 2) ($f$, 1)\} ($g$, 1) ($c$, 3) ($b$, 1)$>$ \\
		\hline  
		\(s_{4}\) & $<$\{($e$, 2) ($f$, 1)\} \{($c$, 1) ($d$,3)\} ($g$, 3) ($b$, 1)$>$ \\
		\hline
	\end{tabular}
\end{table}

\begin{table}[!h]
	\caption{External utility table}
	\label{table2}
	\centering
	\begin{tabular}{|c|c|c|c|c|c|c|c|c|}
		\hline
		\textbf{Item}	    & $a$	& $b$	& $c$	& $d$	& $e$	& $f$  & $g$ & $h$ \\ \hline 
		\textbf{Unit utility}	    & 2     & 1     & 3     & 1     & 2     & 3    & 2    & 1 \\ \hline
		
	\end{tabular}
\end{table}

As the sequence database illustrated in Table \ref{table1}, which will be the running example used in this paper, there are four sequences with $s_1$, $s_2$, $s_3$, and $s_4$ as their \textit{SID}, respectively. In Table \ref{table2}, we can see that the external utility of $a$, $b$, $c$, $d$, $e$, $f$, $g$, and $h$ is 2, 1, 3, 1, 2, 3, 2, and 1, respectively. For example, in $s_2$, there is an item $(a$, $1)$ and we can know that $iu(a)$ $=$ 2 and $q$($a$, $s_2$) $=$ 1. Similarly, the positions of items $a$, $b$, $e$, $f$, $d$, $c$, $h$ in sequence $s_2$ are 1, 1, 2, 2, 3, 4, 5, respectively, and the indices of items $a$, $b$, $e$, $f$, $d$, $c$, $h$ in sequence $s_2$ are 1, 2, 3, 4, 5, 6, 7, respectively.

\begin{definition}[Totally-ordered sequential rule]
	\label{definition:ToSR}
	\rm A totally-ordered sequential rule (ToSR) $r$ $=$ $X$ $\rightarrow$ $Y$, is defined as a relationship between two nonempty sequences $X$ and $Y$, where $X$ is the antecedent of ToSR $r$ and $Y$ is the consequent of ToSR $r$ and $X$ $\cap$ $Y$ $=$ $\varnothing$, that means any item appears in $X$ will not occur in $Y$. For totally-ordered SRM, a ToSR $r$ means that items occur in $Y$ will after the items in $X$ for a given sequence. 
\end{definition}

\begin{definition}[The size of totally-ordered sequential rule ]
	\label{definition:ToSRsize}
	\rm The size of a ToSR $r$ $=$ $X$ $\rightarrow$ $Y$ is denoted as $k$ $\ast$ $m$, where $k$ denotes the number of items that show in the antecedent of $r$ and $m$ denotes the number of items that appear in $r$'s consequent. Note that $k$ $\ast$ $m$ only reveals the length of the antecedent and consequent of $r$. Moreover, a rule $r_1$ with size $g$ $\ast$ $h$ is smaller than rule $r_2$ with size $f$ $\ast$ $l$ if and only if $g$ $\leq$ $f$ and $h$ $\textless$ $l$, or $g$ $\textless$ $f$ and $h$ $\leq$ $l$.
\end{definition}

Take the rule $r_3$ $=$ $<$\{$e$, $f$\}$>$ $\rightarrow$ $<$$c$$>$ and $r_4$ $=$ $<$\{$e$, $f$\}$>$ $\rightarrow$ $<$$c$, $b$$>$, we can get them from Table \ref{table1}, as the example, in which the size of $r_3$ is $2$ $\ast$ $1$ and $r_4$ is $2$ $\ast$ $2$. Thus, it is said that $r_3$ is smaller than $r_4$.

\begin{definition}[Sequence/rule occurrence]
    \rm Given two sequences $s_1$ $=$ $<$$e_1$, $e_2$, $\cdots$, $e_p$$>$ and $s_2$ $=$ $<$$E_1$, $E_2$, $\cdots$, $E_n$$>$, it is said that $s_1$ occurs in $s_2$ (denoted as $s_1$ $\sqsubseteq$ $s_2$) if and only if $\exists$ $1$ $\le$ $j_1$ $<$ $j_2$ $<$ $\ldots$ $<$ $j_p$ $\le$ $n$ such that $e_1$ $\subseteq$ $E_{j_1}$, $e_2$ $\subseteq$ $E_{j_2}$, $\ldots$, $e_p$ $\subseteq$ $E_{j_p}$. A rule $r$ $=$ $X$ $\rightarrow$ $Y$ is said to occur in $s_2$ if and only if there exists an integer $k$ such that $1$ $\le$ $k$ $\textless$ $n$, $X$ $\sqsubseteq$ $<$$E_1$, $E_2$, $\cdots$, $E_k$$>$ and $Y$ $\sqsubseteq$ $<$$E_{k+1}$, $\cdots$, $E_n$$>$. In addition, we denote the set of sequences that contain $r$ as $seq(r)$ and the set of sequences that contain the antecedent as $ant(r)$.
\end{definition}

For example, a ToSR $r_4$ $=$ $<$\{$e$, $f$\}$>$ $\rightarrow$ $<$$c$, $b$$>$ occurs in $s_3$ and $s_4$ and its antecedent occurs in $s_2$, $s_3$, and $s_4$. Therefore, $seq(r)$ and $ant(r)$ are \{$s_3$, $s_4$\} and \{$s_2$, $s_3$, $s_4$\}, respectively.
 
\begin{definition}[Support and confidence]
	\rm Let $r$ be a ToSR and $\mathcal{D}$ be a sequence database. We use the value of $\lvert$$seq(r)$$\lvert$ / $\lvert$$\mathcal{D}$$\rvert$ to represent the support value of ToSR $r$, i.e., \textit{sup}($r$). It implies that the number of sequences containing ToSR $r$ divided by the total number of sequences in $\mathcal{D}$ gives $r$'s support value. The confidence of ToSR $r$ is defined as \textit{conf}($r$) $=$ $\lvert$$seq(r)$$\lvert$/$\lvert$$ant(r)$$\rvert$, which means that the confidence value of ToSR $r$ equals the number of sequences that $r$ appear divides by the number of sequences in which antecedent $X$ appears.
\end{definition}

\begin{definition}[Utility of an item/itemset in a sequence]	
	\rm Given an item $i$, an itemset $I$, and a sequence $s_k$, the utility of an item is equal to its internal utility multiplies its external utility. Let $u$($i$, $s_k)$ denotes the the utility of item $i$ in sequence $s_k$ and is defined as $q(i$, $s_k)$ $\times$ $iu(i)$. The utility of the itemset $I$ in the sequence $s_k$ is designated as $u(I$, $s_k)$ and defined as $u(I$, $s_k)$ $=$ $\sum_{i \in I}$ $q(i,s_k)$ $\times$ $iu(i)$. 
\end{definition}

\begin{definition}[Utility of a totally-ordered sequential rule in a sequence]	
	\rm Let $r$ be a ToSR and $s_k$ be a sequence. We use $u$($r$, $s_k)$ to represent the utility of ToSR $r$ in sequence $s_k$. Then $u$($r$, $s_k)$ is defined as $u$($r$, $s_k)$ $=$ $\sum_{i \in r \land s_k \subseteq \textit{seq}(r) }$ $q(i, s_k)$ $\times$ $iu(i)$.
\end{definition}

\begin{definition}[Utility of a totally-ordered sequential rule in a database]	
	\rm Given a ToSR $r$ and a sequence database $\mathcal{D}$, we use $u(r)$ to denote the utility of ToSR $r$ in the sequence database $\mathcal{D}$. Then $u(r)$ is defined as $u(r)$ $=$ $\sum_{s_k \in \textit{seq}(r) \land \textit{seq}(r) \subseteq \mathcal{D}}$ $u$($r$, $s_k)$.
\end{definition}

For example, a ToSR $r_4$ $=$ $<$\{$e$, $f$\}$>$ $\rightarrow$ $<$$c$, $b$$>$ occurs in $s_3$ and $s_4$, and its $seq(r_4)$ $=$ \{$s_3$, $s_4$\} and $ant(r_4)$ $=$ \{$s_2$, $s_3$, $s_4$\}. Thus, the support value of rule $r_4$ is $sup(r_4)$ $=$ $\lvert$$seq(r_4)$$\lvert$ / $\lvert$$\mathcal{D}$$\rvert$ $=$ 2 / 4 $=$ 0.5, and the confidence value of rule $r_4$ is $conf(r_4)$ $=$ $\lvert$$seq(r_4)$$\lvert$/$\lvert$$ant(r_4)$$\rvert$ $=$ 2 / 3 $=$ 0.67. The utility of item $e$ in sequence $s_3$ is $u(e, s_3)$ $=$ $q(e, s_3)$ $\times$ $iu(e)$ $=$ 2 $\times$ 2 $=$ 4 and the utility of rule $r_4$ in sequence $s_3$ is $u(r_4, s_3)$ $=$ $\sum_{i \in r \land s_k \subseteq \textit{seq}(r)}$ $q(i, s_k)$ $\times$ $iu(i)$ $=$ 2 $\times$ 2 $+$ 1 $\times$ 3 $+$ 3 $\times$ 3 $+$ 1 $\times$ 1 $=$ 4 $+$ 3 $+$ 9 $+$ 1 $=$ 17. Correspondingly, the utility of $r_4$ in sequence $s_4$ is $u(r_4, s_4)$ $=$ 11. Therefore, the utility of rule $r_4$ is $u(r_4)$ $=$ $\sum_{s_k \in \textit{seq}(r_4) \land \textit{seq}(r_4) \subseteq \mathcal{D}}$ $u$($r_4$, $s_k$) $=$ $u$($r_4$, $s_3$) $+$ $u$($r_4$, $s_4$) $=$ 17 $+$ 11 $=$ 28.

\subsection{Problem description}

\begin{definition}[High-utility totally-ordered sequential rule mining]	
	\rm Given a sequence database $\mathcal{D}$, a positive minimum utility threshold \textit{minutil} and a minimum confidence threshold \textit{minconf} between 0 and 1, a ToSR is called high-utility totally-ordered sequential rule (HTSR) if and only if it satisfies both the minimum utility and confidence thresholds simultaneously, i.e., $u(r)$ $\ge$ \textit{minutil} and \textit{conf}$(r)$ $\ge$ \textit{minconf}. Thus, the problem of high-utility totally-ordered sequential rule mining is to identify and output all ToSRs that satisfy both the conditions of \textit{minutil} and \textit{minconf}. Note that in this paper we use ToSR to represent the candidate HTSR.
\end{definition}

\begin{table}[h]
	\centering
	\caption{HTSRs in Table \ref{table1} when \textit{minutil} $=$ 25 and \textit{minconf} $=$ 0.5}
	\label{table3}
	
	\begin{tabular}{|c|c|c|c|c|}  
		\hline 
		\textbf{ID} & \textbf{HTSR} & \textbf{Support} & \textbf{Confidence} & \textbf{Utility}\\
		\hline 
		\(r_{1}\) & $<$\{$e$, $f$\}, $c$$>$ $\rightarrow$ $<$$b$$>$ & 0.5 & 0.67 & 28\\ 
		\hline 
		\(r_{2}\) & $<$$e$$>$ $\rightarrow$ $<$$c$$>$ & 0.75 & 1.0 & 25\\
		\hline 
		\(r_{3}\) & $<$\{$e$, $f$\}$>$ $\rightarrow$ $<$$c$$>$ & 0.75 & 1 & 34\\
		\hline 
		\(r_{4}\) & $<$\{$e$, $f$\}$>$ $\rightarrow$ $<$$c$, $b$$>$ & 0.5 & 0.67 & 28\\
		\hline
	\end{tabular}
\end{table}

For example, if we specify that \textit{minutil} $=$ 25 and \textit{minconf} $=$ 0.5, we will discover four HTSRs shown in Table \ref{table3}. From the result, we can find that two ToSRs $<$\{$a$, $b$\}$>$ $\rightarrow$ $<$$c$, $d$$>$ and $<$\{$a$, $b$\}$>$ $\rightarrow$ $<$$d$, $c$$>$ with utility 15 and 10, respectively. Both are low-utility and will not be output. Using a partially-ordered SRM algorithm, the two ToSRs will be merged as high-utility SR \{$a$, $b$\} $\rightarrow$ \{$c$, $d$\} with a utility of 25. However, it is not committed to reality.

\begin{definition}[I-expansion and S-expansion]
	\label{EToSR}
	\rm Let $s$ $=$ $<$$e_1$, $e_2$, $\cdots$, $e_k$$>$ be a sequence and $i$ $\in$ $I$ be an item. The I-expansion is defined as $<$$e_1$, $e_2$, $\cdots$, $e_k$ $\cup$ \{$i$\}$>$, where $i$ should be occurring simultaneously with items in itemset $e_k$ and greater than the items in $e_k$ according to the $\succ_{lex}$. Given a sequence $s$ $=$ $<$$e_1$, $e_2$, $\cdots$, $e_k$$>$ and an item $i$ $\in$ $I$, the S-expansion is defined as $<$$e_1$, $e_2$, $\cdots$, $e_k$, \{$i$\}$>$, where $i$ should be occurring after the items in itemset $e_k$ of a sequence.
\end{definition}

\begin{definition}[The expansion of a totally-ordered sequential rule]
	\rm Similar to RuleGrowth \cite{fournier2015mining}, in this paper, TotalSR implements left and right expansion to grow a ToSR. Given a rule $r$ $=$ $X$ $\rightarrow$ $Y$, where $X$ $=$ $<$$e_1$, $e_2$, $\cdots$, $e_k$$>$ and $Y$ $=$ $<$$e_m$, $e_{m+1}$, $\cdots$, $e_n$$>$ ($k$ $\textless$ m $\leq$ $n$), the left expansion is defined as $<$$e_1$, $e_2$, $\cdots$, $e_k$$>$ $\diamondsuit$ $i$ $\rightarrow$ $<$$e_m$, $e_{m+1}$, $\cdots$, $e_n$$>$, where $\diamondsuit$ represents the expansion can be both I-expansion and S-expansion and item $i$ should not be in $Y$, i.e., $i$ $\notin$ $Y$. Correspondingly, the right expansion is defined as $<$$e_1$, $e_2$, $\cdots$, $e_k$$>$ $\rightarrow$ $<$$e_m$, $e_{m+1}$, $\cdots$, $e_n$$>$ $\diamondsuit i$, where $\diamondsuit$ represents the expansion can be both I-expansion and S-expansion and item $i$ $\notin$ $X$.
\end{definition}

A ToSR can be formed by first performing a left expansion and then a right expansion or performing a right expansion and then a left expansion. To avoid the repetition generating of the same rule, in TotalSR, unlike \cite{zida2015efficient, huang2021us} we stipulate that a ToSR cannot implement the left expansion after it performs a right expansion, which means a left-first expansion.

\section{The Proposed Algorithm} 
\label{sec:techs}

Two novel algorithms, TotalSR and TotalSR$^+$, will be presented in this section. The corresponding data structures, pruning strategies, and main procedures of TotalSR and TotalSR$^+$ will be described in this section, respectively.

\subsection{Upper bounds and pruning strategies}

In this paper, we also use some extraordinary techniques that are utilized in the utility-oriented pattern mining field. Sequence estimated utility (\textit{SEU}) can help us to only keep useful items in the database, which can be referred to \cite{zida2015efficient} to get the detailed description. Left and right expansion prefix extension utilities are inspired by prefix extension utilities (\textit{PEU}). The left and right reduced sequence utilities are motivated by the reduced sequence utility (\textit{RSU}). Also, the left and right reduced sequence prefix extension utilities are designed by the combination of \textit{PEU} and \textit{RSU}. All the specific definitions of \textit{PEU} and \textit{RSU} can refer to the former studies \cite{yin2012uspan, wang2016efficiently, gan2020proum, truong2019survey}.

\begin{definition}[Sequence estimated utility of item/ToSR]
	\label{SEU}
	\rm Let $a$ be an item and $\mathcal{D}$ be a sequence database. The sequence estimated utility (\textit{SEU}) of item $a$ is designated as \textit{SEU}($a$) and defined as \textit{SEU}($a$) $=$ $\sum_{a \in s_k \land i \in s_k \land s_k \in \mathcal{D}}$ \textit{u}($i$, $s_k$), where $s_k$ is the sequence that contains item $a$ and $i$ is the item that appear in sequence $s_k$. Note that $i$ can be any item in the sequence $s_k$. Correspondingly, given a ToSR $r$, the sequence estimated utility of $r$ is denoted as \textit{SEU}($r$) and defined as \textit{SEU}($r$) $=$ $\sum_{i \in s_k \land s_k \in seq(r)}$ \textit{u}($i$, $s_k$), where $i$ is the item that occurs in the sequence $s_k$.
\end{definition}

\begin{definition}[Promising item and promising ToSR]
	\rm A promising item is the one whose \textit{SEU} is no less than \textit{minutil}. In other words, for each promising item $i_k$, we have \textit{SEU}($i_k$) $\ge$ \textit{minutil}. If the \textit{SEU} of an item $i$ is less than \textit{minutil}, it implies that item $i$ is unpromising. Also, a promising ToSR is the one whose \textit{SEU} is no less than \textit{minutil}. In other words, for each promising ToSR $r_k$, we have \textit{SEU}($r_k$) $\ge$ \textit{minutil}. Contrarily, if the \textit{SEU} of a ToSR $r$ is smaller than \textit{minutil}, it means that $r$ is unpromising.
\end{definition}

In Table \ref{table1}, for instance, if we set \textit{minutil} $=$ 20, the \textit{SEU} of $a$ is \textit{SEU}($a$) $=$ $\sum_{a \in s_k \land i \in s_k \land s_k \in \mathcal{D}}$ \textit{u}($i$, $s_k$) $=$ $\sum_{a \in s_1 \land i \in s_1 \land s_1 \in \mathcal{D}}$ \textit{u}($i$, $s_1$) $+$ $\sum_{a \in s_2 \land i \in s_2 \land s_2 \in \mathcal{D}}$ \textit{u}($i$, $s_2$) $=$ 21 $+$ 16 $=$ 37, then item $a$ is a promising item. However, the \textit{SEU} of item $h$ is \textit{SEU}($h$) $=$ $\sum_{h \in s_k \land i \in s_k \land s_k \in \mathcal{D}}$ \textit{u}($i$, $s_k$) $=$ $\sum_{h \in s_2 \land i \in s_2 \land s_2 \in \mathcal{D}}$ \textit{u}($i$, $s_2$) $=$ 16. Thus, there is an unpromising item $h$. For a ToSR $r$ $=$ $<$$a$$>$ $\rightarrow$ $<$$h$$>$, its \textit{SEU} is \textit{SEU}($r$) $=$ $\sum_{i \in s_k \land s_k \in seq(r)}$ \textit{u}($i$, $s_k$) $=$ $\sum_{i \in s_2}$ \textit{u}($i$, $s_2$) $=$ 16, and it is an unpromising rule. Like the Reference \cite{zida2015efficient}, we have the following two strategies.

\begin{strategy}[Unpromising items pruning strategy]
	\label{Strategy:UIP}
	\rm Given a sequence database $\mathcal{D}$, TotalSR will remove all unpromising items from $\mathcal{D}$. For an unpromising item $i$, the \textit{SEU} of $i$ is smaller than \textit{minutil}. As a result, the utility of any ToSR that contains item $i$ will not exceed the \textit{minutil}. In other words, the unpromising item $i$ will not be contained in a HTSR, which means the item $i$ is useless for HTSRs. We can remove the item from $\mathcal{D}$ directly. Similar to US-Rule \cite{huang2021us}, after removing some items from $\mathcal{D}$, the other items' \textit{SEU} will be changed. Thus, TotalSR will not stop using the unpromising item pruning (UIP) strategy until no items are removed.
\end{strategy}

\begin{strategy}[Unpromising sequential rules pruning strategy]
	\label{Strategy:USRP}
	\rm Let $r$ be an unpromising ToSR. The unpromising sequential rules pruning strategy is that TotalSR will not extend $r$ further. For an unpromising ToSR $r$, the \textit{SEU} of $r$ is less than \textit{minutil}. As a result, it is not exceeded \textit{minutil} for any ToSR containing $r$.
\end{strategy}

Inspired of the upper bound \textit{PEU} commonly used in HUSPM \cite{yin2012uspan, wang2016efficiently, gan2020fast, gan2020proum, truong2019survey}, in this paper, we design two analogous upper bounds called \textit{LEPEU} and \textit{REPEU}. They are originally proposed in \cite{huang2021us}.

\begin{upper bound}[Left expansion prefix extension utility]
    \label{Upper bound: LEPEU}
	\rm We use \textit{LEPEU}($r$, $s$) to represent the left expansion prefix extension utility (\textit{LEPEU}) of a ToSR $r$ in sequence $s$. \textit{LEPEU}($r$, $s$) is defined as:
	$$ \textit{LEPEU}(r, s) = \left\{
	\begin{aligned}
	u(r, s) + \textit{ULeft}(r, s),&  & \textit{ULeft}(r, s) \textgreater 0 \\
	0, & & otherwise.
	\end{aligned}
	\right.
	$$
\end{upper bound}

Note that \textit{ULeft}($r$, $s$) denotes the sum of utility of the items in a sequence $s$ that can be extended into the left part of a ToSR, i.e., extended into antecedent. Correspondingly, we use \textit{LEPEU}($r$) to denote the \textit{LEPEU} of a ToSR in $\mathcal{D}$. Then, \textit{LEPEU}($r$) is defined as:

\begin{center}
	\textit{LEPEU}($r$) $=$ $\sum_{s_k \in \textit{seq}(r) \land \textit{seq}(r) \subseteq \mathcal{D}}$ \textit{LEPEU}($r$, $s_k$).
\end{center}

In Table \ref{table3}, for instance, consider the HTSR $r_3$ $=$ $<$\{$e$, $f$\}$>$ $\rightarrow$ $<$$c$$>$. There are three sequences, $s_2$, $s_3$, and $s_4$, that contain $r_3$. In sequence $s_2$, item $d$ can be extended into the antecedent of $r_3$, so \textit{LEPEU}($r_3$, $s_2$) $=$ 10. Also, we have \textit{LEPEU}($r_3$, $s_3$) $=$ 18. However, in $s_4$, there is no item to be extended. Therefore, \textit{LEPEU}($r_3$, $s_4$) $=$ 0. Finally, \textit{LEPEU}($r_3$) $=$ \textit{LEPEU}($r_3$, $s_2$) $+$ \textit{LEPEU}($r_3$, $s_3$) $+$ \textit{LEPEU}($r_3$, $s_4$) $=$ 10 $+$ 18 $+$ 0 $=$ 28.

\begin{theorem}
	\label{Theorem:LEPEU}
	Given a ToSR $r$ $=$ $X$ $\rightarrow$ $Y$ and the other ToSR $r^\prime$, where $r^\prime$ is extended from $r$ by performing a left I- or S-expansion, we have $u$($r^\prime$) $\le$ \textit{LEPEU}($r$). 
\end{theorem}

\begin{proof}
	\label{Proof:LEPEU}
	\rm Given an item $i$, a ToSR $r$, and a sequence $s$, where $i$ can be extended into the antecedent of $r$ to form the other ToSR $r^\prime$, according to the definition of \textit{ULeft}($r$, $s$) in upper bound \textit{LEPEU}, we have $u$($i$, $s$) $\le$ \textit{ULeft}($r$, $s$). Then, we have $u$($r^\prime$, $s$) $=$ $u$($r$, $s$) + $u$($i$, $s$) $\le$ $u$($r$, $s$) +\textit{ULeft}($r$, $s$) $=$ \textit{LEPEU}($r$, $s$). Therefore, $u$($r^\prime$) $\le$ \textit{LEPEU}($r$).
\end{proof}

\begin{upper bound}[Right expansion prefix extension utility]
    \label{Upper bound: REPEU}
	\rm We use \textit{REPEU}($r$, $s$) to denote the right expansion prefix extension utility (\textit{REPEU}) of a ToSR $r$ in sequence $s$. \textit{REPEU}($r$, $s$) is defined as:
	$$ \textit{REPEU}(r, s)=\left\{
	\begin{aligned}
	u(r, s) + \textit{URight}(r, s),  &  & \textit{URight}(r, s)  \textgreater 0 \\
	0, & & otherwise.
	\end{aligned}
	\right.
	$$
\end{upper bound}

Note that \textit{URight}($r$, $s$) represents the sum of utility of the items in sequence $s$ that can be extended into the right part of a ToSR, i.e., extended into the consequent. Correspondingly, we use \textit{REPEU}($r$) to denote the \textit{REPEU} of $r$ in $\mathcal{D}$. \textit{REPEU}($r$) is defined as: 
\begin{center}
	\textit{REPEU}($r$) $=$ $\sum_{s_k \in \textit{seq}(r) \land \textit{seq}(r) \subseteq \mathcal{D}}$ \textit{REPEU}($r$, $s_k$).
\end{center}

In Table \ref{table3}, for instance, consider the HTSR $r_3$ $=$ $<$\{$e$, $f$\}$>$ $\rightarrow$ $<$$c$$>$. HTSR $r_3$ occurs in sequences $s_2$, $s_3$, and $s_4$. In sequence $s_2$, there is no item to be extended into the consequent of $r_3$ (note that item $h$ is removed from $\mathcal{D}$ since it is an unpromising item). Thus \textit{REPEU}($r_3$, $s_2$) $=$ 0. In sequence $s_3$, there is an item $b$ can be extended into the consequent of rule $r_3$. Thus, \textit{REPEU}($r_3$, $s_3$) $=$ 17. And in $s_4$, \textit{REPEU}($r_3$, $s_4$) $=$ 20. Finally, \textit{REPEU}($r_3$) $=$ \textit{REPEU}($r_3$, $s_2$) $+$ \textit{REPEU}($r_3$, $s_2$) $+$ \textit{REPEU}($r_3$, $s_4$) $=$ 0 $+$ 17 $+$ 20 $=$ 37.

\begin{theorem}
	\label{Theorem:REPEU}
	Given a ToSR $r$ $=$ $X$ $\rightarrow$ $Y$ and the other ToSR $r^\prime$, where $r^\prime$ is extended from $r$ by performing a right I- or S-expansion, we have $u$($r^\prime$) $\le$ \textit{REPEU}($r$). 
\end{theorem}

\begin{proof}
	\label{Proof:REPEU}
	\rm Given a sequence $s$, a ToSR $r$, and an item $i$ that can be extended into the consequent of $r$ to form rule $r^\prime$, according to the definition of \textit{URight}($r$, $s$) in upper bound \textit{REPEU}, we have $u$($i$, $s$) $\le$ \textit{URight}($r$, $s$). Then, we have $u$($r^\prime$, $s$) $=$ $u$($r$, $s$) $+$ $u$($i$, $s$) $\le$ $u$($r$, $s$) $+$\textit{URight}($r$, $s$) $=$ \textit{REPEU}($r$, $s$). Therefore, $u$($r^\prime$) $\le$ \textit{REPEU}($r$).
\end{proof}

Inspired of the upper bound \textit{RSU} used in HUSPM \cite{wang2016efficiently, gan2020fast, gan2020proum, truong2019survey}, in this paper, we use two homologous upper bounds proposed in \cite{huang2021us}, called \textit{LERSU} and \textit{RERSU}.

\begin{upper bound}[Left expansion reduced sequence utility]
    \label{Upper bound: LERSU}
	\rm Let $\xi$ be a ToSR that can perform a left expansion with an item $w$ to form the other ToSR $r$. In a sequence $s$, we use \textit{LERSU}($r$, $s$) to represent the left expansion reduced sequence utility (\textit{LERSU}) of $r$. Thus, \textit{LERSU}($r$, $s$) is defined as:
	$$ \textit{LERSU}(r, s)=\left\{
	\begin{aligned}
	\textit{LEPEU}(\xi, s),   &  & s \in \textit{seq}(r) \\
	0, & & otherwise.
	\end{aligned}
	\right.
	$$
\end{upper bound}

Correspondingly, we use \textit{LERSU}($r$) to present the \textit{LERSU} of $r$ in $\mathcal{D}$. Then, \textit{LERSU}($r$) is defined as:
\begin{center}
	\textit{LERSU}($r$) $=$ $\sum_{\forall s \in \mathcal{D}}$ \textit{LERSU}($r, s$).
\end{center}

For example, in Table \ref{table3}, consider the HTSR $r_2$ $=$ $<$$e$$>$ $\rightarrow$ $<$$c$$>$ and $r_3$ $=$ $<$\{$e$, $f$\}$>$ $\rightarrow$ $<$$c$$>$. HTSR $r_3$ can be formed with the item $f$ by performing a left I-expansion from $r_2$. Moreover, \textit{seq}($r_3$) $=$ \{$s_2$, $s_3$, $s_4$\}. Thus, \textit{LERSU}($r_3$) $=$ \textit{LEPEU}($r_2$, $s_2$) $+$ \textit{LEPEU}($r_2$, $s_3$) $+$ \textit{LEPEU}($r_2$, $s_4$) $=$ 10 $+$ 18 $+$ 10 $=$ 38.

\begin{theorem}
	\label{Theorem:LERSU}
	Given a ToSR $r$ $=$ $X$ $\rightarrow$ $Y$ and the other ToSR $r^\prime$, where $r^\prime$ is extended with item $i$ from $r$ by performing a left I- or S-expansion, we have $u$($r^\prime$) $\le$ \textit{LERSU}($r$). 
\end{theorem}

\begin{proof}
	\label{Proof:LERSU}
	\rm Let $s$ be a sequence, $r$ be a ToSR, and $i$ be an item that can be extended into the antecedent of $r$ to form rule $r^\prime$. According to the definition of upper bound \textit{LERSU} and theorem \ref{Theorem:LEPEU}, each sequence $s$ $\in$ $seq(r^\prime$), we have $u$($r^\prime$, $s$) $\leq$ \textit{LEPEU}($r^\prime$, $s$) and \textit{LEPEU}($r^\prime$, $s$) $\le$ \textit{LERSU}($r$, $s$). Thus $u$($r^\prime$, $s$) $\le$ \textit{LERSU}($r$, $s$). Finally, $\sum_{\forall s \in \mathcal{D}}u(r^\prime, s) \le \sum_{\forall s \in \mathcal{D}}$\textit{LERSU}($r$, $s$), i.e., $u$($r^\prime$) $\le$ \textit{LERSU}($r$). 
\end{proof}

\begin{upper bound}[Right expansion reduced sequence utility]
    \label{Upper bound: RERSU}
	\rm Let $\xi$ be a ToSR that can implement a right expansion to form the other ToSR $r$. In a sequence $s$, we use \textit{RERSU}($r$, $s$) to designate the right expansion reduced sequence utility (\textit{RERSU}) of $r$. Then, \textit{RERSU}($r$, $s$) is defined as:
	$$ \textit{RERSU}(r, s)=\left\{
	\begin{aligned}
	\textit{REPEU}(\xi, s),   &  & s \in \textit{seq}(r) \\
	0, & & otherwise.
	\end{aligned}
	\right.
	$$
\end{upper bound}

Correspondingly, we use \textit{RERSU}($r$) to denote the \textit{RERSU} of $r$ in a sequence database $\mathcal{D}$. \textit{RERSU}($r$) is defined as:
\begin{center}
	\textit{RERSU}($r$) $=$ $\sum_{\forall s \in \mathcal{D}}$\textit{RERSU}($r$, $s$).
\end{center}

For example, consider the HTSR $r_3$ $=$ $<$\{$e$, $f$\}$>$ $\rightarrow$ $<$$c$$>$ and $r_4$ $=$ $<$\{$e$, $f$\}$>$ $\rightarrow$ $<$$c$, $b$$>$. HTSR $r_4$ can be formed with the item $b$ by performing a right S-expansion from $r_3$. Besides, \textit{seq}($r_4$) $=$ \{$s_3$, $s_4$\}. According to upper bound \textit{RERSU}, \textit{RERSU}($r_4$) $=$ \textit{REPEU}($r_3$, $s_3$) $+$ \textit{REPEU}($r_3$, $s_4$) $=$ 17 $+$ 20 $=$ 37. 

\begin{theorem}
	\label{Theorem:RERSU}
	Given a ToSR $r$ $=$ $X$ $\rightarrow$ $Y$ and the other ToSR $r^\prime$, where $r^\prime$ is extended with an item $i$ from $r$ by performing a right I- or S-expansion, we have $u$($r^\prime$) $\le$ \textit{RERSU}($r$). 
\end{theorem}

\begin{proof}
	\label{Proof:RERSU}
	\rm Let $s$ be a sequence, $r$ be a ToSR, and $i$ be an item that can be extended into the consequent of $r$ to form rule $r^\prime$. According to the definition of upper bound \textit{RERSU} and theorem \ref{Theorem:REPEU}, each sequence $s$ $\in$ \textit{seq}($r^\prime$), we can know that $u$($r^\prime$, $s$) $\leq$ \textit{REPEU}($r^\prime$, $s$) and \textit{REPEU}($r^\prime$, $s$) $\le$ \textit{RERSU}($r$, $s$). Thus $u$($r^\prime$, $s$) $\le$ \textit{RERSU}($r$, $s$). Finally, $\sum_{\forall s \in \mathcal{D}}$$u$($r^\prime$, $s$) $\le$ $\sum_{\forall s \in \mathcal{D}}$\textit{RERSU}($r$, $s$), i.e., $u$($r^\prime$) $\le$ \textit{RERSU}($r$).
\end{proof}

Since when we use \textit{LERSU} (\textit{RERSU}), it just utilizes the corresponding \textit{LEPEU} (\textit{REPEU}) of the sequences that can extend with a specific item $i$. However, there is a little useless utility in the remaining utility. Thus, we design two tighter upper bounds that can be viewed as a combination of \textit{RSU} and \textit{PEU}. The reasons why we combine \textit{RSU}and \textit{PEU} to design these two upper bounds are from two aspects. One (\textit{RSU} aspect) is from reduced sequence since it satisfies the anti-monotonic property when we extend a ToSR $r$ with an item $i$, i.e., \textit{seq}($r^\prime$) $\subseteq$ \textit{seq}($r$), where $r^\prime$ is extended from $r$ with an item $i$. The other (\textit{PEU} aspect) is that the remaining utility that can be useful for a ToSR generation is usually less than the total remaining utility. Because the utility from item $i$ to the last extendable item is the useful utility for rule growth.

\begin{upper bound}[Left expansion reduced sequence prefix extension utility]
    \label{Upper bound: LERSPEU}
	\rm For a ToSR $r$ that can implement a left expansion with an item $i$ to generate the other ToSR $r^\prime$, \textit{LERSPEU}($r^\prime$, $s$) denotes the left expansion reduced sequence prefix extension utility (\textit{LERSPEU}) of $r^\prime$ in sequence $s$. Thus, \textit{LERSPEU}($r^\prime$, $s$) is defined as:
	$$ \textit{LERSPEU}(r^\prime, s)=\left\{
	\begin{aligned}
	u(r, s) + \textit{UILeft}(r, i, s),  &  & s  \in seq(r^\prime) \\
	0, & & otherwise.
	\end{aligned}
	\right.
	$$
\end{upper bound}

Where \textit{UILeft}($r$, $i$, $s$) implies the total utility from item $i$ to the last item in sequence $s$ that can be extended into the left part of a ToSR, i.e., extended into antecedent. Note that the last item in sequence $s$ that can be extended into the antecedent is not the end item of sequence $s$. Correspondingly, we use \textit{LERSPEU}($r^\prime$) to signify the \textit{LERSPEU} of a ToSR $r^\prime$ in $\mathcal{D}$. Accordingly, \textit{LERSPEU}($r^\prime$) is defined as:
\begin{center}
	\textit{LERSPEU}($r^\prime$) $=$ $\sum_{s_k \in seq(r^\prime) \land seq(r^\prime) \subseteq \mathcal{D}}$ \textit{LERSPEU}($r^\prime$, $s_k$).
\end{center}

For example, consider the HTSR $r_1$ $=$ $<$\{$e$, $f$\}, $c$$>$ $\rightarrow$ $<$$b$$>$ in Table \ref{table3} and a ToSR $r$ $=$ $<$\{$e$, $f$\}$>$ $\rightarrow$ $<$$b$$>$. HTSR $r_1$ can be generated from $r$ by extending an item $c$ to its antecedent and \textit{seq}($r_1$) $=$ \{$s_3$, $s_4$\}. In $s_3$ there exists a useless item $g$. Thus \textit{LERSPEU}($r_1$, $s_3$) $=$ 8 $+$ 9 $=$ 17. In $s_4$, there is no useless item, so \textit{LERSPEU}($r_1$, $s_4$) $=$ 8 $+$ 12 $=$ 20. In total, \textit{LERSPEU}($r_1$) $=$ \textit{LERSPEU}($r_1$, $s_3$) $+$ \textit{LERSPEU}($r_1$, $s_4$) $=$ 17 $+$ 20 $=$ 37.

\begin{theorem}
	\label{Theorem:LERSPEU}
	Given a ToSR $r$ $=$ $X$ $\rightarrow$ $Y$ and the other ToSR $r^\prime$, where $r^\prime$ is extended with an item $i$ from $r$ by performing a left I- or S-expansion, we have $u$($r^\prime$) $\le$ \textit{LERSPEU}($r^\prime$). 
\end{theorem}

\begin{proof}
	\label{Proof:LERSPEU}
	\rm Let $s$ be a sequence, $r$ be a ToSR, and $i$ be an item that can be extended into the antecedent of $r$ to form rule $r^\prime$. According to the definition of \textit{UILeft}($r$, $i$, $s$) in upper bound \textit{LERSPEU}, we have $u$($i$, $s$) $\le$ \textit{UILeft}($r$, $i$, $s$). Then, we have $u$($r^\prime$, $s$) $=$ $u$($r$, $s$) $+$ $u$($i$, $s$) $\le$ $u$($r$, $s$) $+$ \textit{UILeft}($r$, $i$, $s$) $=$ \textit{LERSPEU}($r^\prime$, $s$). Therefore, $u$($r^\prime$) $\le$ \textit{LERSPEU}($r^\prime$).
\end{proof}

\begin{upper bound}[Right expansion reduced sequence prefix extension utility]
    \label{Upper bound: RERSPEU}
	\rm For a ToSR $r$ that can implement a right expansion with an item $i$ to produce the other ToSR $r^\prime$, \textit{RERSPEU}($r^\prime$, $s$) denotes the right expansion reduced sequence prefix extension utility (RERSPEU) of $r^\prime$ in sequence $s$. \textit{RERSPEU}($r^\prime$, $s$) is defined as:
	$$ \textit{RERSPEU}(r^\prime, s)=\left\{
	\begin{aligned}
	u(r, s) + \textit{UIRight}(r, i, s),  &  & s  \in seq(r^\prime) \\
	0, & & otherwise.
	\end{aligned}
	\right.
	$$
\end{upper bound}

Where \textit{UIRight}($r$, $i$, $s$) denotes the total utility from item $i$ to the last item in sequence $s$ because all item after item $i$ can be extended into the ToSR. Thus, the \textit{RERSPEU} of a ToSR $r^\prime$ in $\mathcal{D}$ is denoted as \textit{RERSPEU}($r^\prime$) and defined as:
\begin{center}
	\textit{RERSPEU}($r^\prime$) $=$ $\sum_{s_k \in \textit{seq}(r^\prime) \land \textit{seq}(r^\prime) \subseteq \mathcal{D}}$ \textit{RERSPEU}($r^\prime$, $s_k$).
\end{center}

For example, consider the HTSR $r_3$ $=$ $<$\{$e$, $f$\}$>$ $\rightarrow$ $<$$c$$>$ and $r_4$ $=$ $<$\{$e$, $f$\}$>$ $\rightarrow$ $<$$c$, $b$$>$ in Table \ref{table3}. HTSR $r_4$ can be generated from $r_3$ by extending an item $b$ to its consequent and \textit{seq}($r_4$) $=$ \{$s_3$, $s_4$\}. In $s_3$ there is no useless item. Thus \textit{RERSPEU}($r_4$, $s_3$) $=$ 16 $+$ 1 $=$ 17. However, in $s_4$, there are two items, $d$ and $g$, before item $b$, i.e., they are useless. Therefore, \textit{RERSPEU}($r_4$, $s_4$) $=$ 10 $+$ 1 $=$ 11. In total, \textit{RERSPEU}($r_4$) $=$ \textit{RERSPEU}($r_4$, $s_3$) $+$ \textit{RERSPEU}($r_4$, $s_4$) $=$ 17 $+$ 11 $=$ 28.

\begin{theorem}
	\label{Theorem:RERSPEU}
	Given a ToSR $r$ $=$ $X$ $\rightarrow$ $Y$ and the other ToSR $r^\prime$, where $r^\prime$ is extended with an item $i$ from $r$ by performing a right I- or S-expansion, we have $u$($r^\prime$) $\le$ \textit{RERSPEU}($r^\prime$). 
\end{theorem}

\begin{proof}
	\label{Proof:RERSPEU}
	\rm Let $s$ be a sequence, $r$ be a ToSR, and $i$ be an item that can be extended into the consequent of $r$ to form rule $r^\prime$. According to the definition of \textit{UIRight}($r$, $i$, $s$) in upper bound \textit{RERSPEU}, we have $u$($i$, $s$) $\le$ \textit{UIRight}($r$, $i$, $s$). Then, we have $u$($r^\prime$, $s$) $=$ $u$($r$, $s$) $+$ $u$($i$, $s$) $\le$ $u$($r$, $s$) $+$ \textit{UIRight}($r$, $i$, $s$) $=$ \textit{RERSPEU} ($r^\prime$, $s$). Therefore, $u$($r^\prime$) $\le$ \textit{RERSPEU}($r^\prime$).
\end{proof}

Since in the utility-oriented mining area, there is no ideally anti-monotonic property. To avoid the severe combinatorial explosion problem of the search space when we set a lower \textit{minutil}, we propose several utility upper bound pruning strategies to cope with the combinatorial explosion problem. In addition, since confidence value is not related to utility, it is only generated from support value. In other words, the confidence value of a ToSR has an anti-monotonic property. Thus, we can design a pruning strategy based on confidence. However, it is required that a left-first expansion be made to ensure the anti-monotonic property can work correctly. The pruning strategies are given below. Note that the pruning strategies \ref{Strategy:LEPEU} to \ref{Strategy:RERSU} are originated from \cite{huang2021us}.

\begin{strategy}[Left expansion prefix extension utility pruning strategy]
	\label{Strategy:LEPEU}
	\rm Given a ToSR $r$, according to the Theorem \ref{Theorem:LEPEU}, when $r$ implements a left expansion, the utility of $r$ will not exceed the upper bound \textit{LEPEU}, i.e., $u$($r$) $\le$ \textit{LEPEU}($r$). If \textit{LEPEU}($r$) $\textless$ \textit{minutil}, we can know that any ToSR extending from $r$ by implementing a left expansion will not be a HTSR, i.e., $u$($r^\prime$) $\textless$ \textit{minutil}. Thus, we can stop to extend further.
\end{strategy}

\begin{strategy}[Right expansion prefix extension utility pruning strategy]
	\label{Strategy:REPEU}
	\rm Given a ToSR $r$, according to the Theorem \ref{Theorem:REPEU}, when $r$ implements a right expansion, the utility of $r$ will not exceed the upper bound \textit{REPEU}, i.e., $u$($r$) $\le$ \textit{REPEU}($r$). If \textit{REPEU}($r$) $\textless$ \textit{minutil}, we can know that any ToSR extending from $r$ by implementing a right expansion will not be a HTSR, i.e., $u$($r^\prime$) $\textless$ \textit{minutil}. Thus, we can stop to extend further.
\end{strategy}

\begin{strategy}[Left expansion reduced sequence utility pruning strategy]
	\label{Strategy:LERSU}
	\rm Given a ToSR $r$, according to the Theorem \ref{Theorem:LERSU}, when $r$ implements a left expansion with a specific item $i$ to generate a ToSR $r^\prime$, the utility of $r^\prime$ will not exceed the upper bound \textit{LERSU}, i.e., $u$($r^\prime$) $\le$ \textit{LERSU}($r$). If \textit{LERSU}($r$) $\textless$ \textit{minutil}, we can know that any ToSR extending from $r$ by implementing a left expansion will not be a HTSR. Thus, we can stop extending further.
\end{strategy}

\begin{strategy}[Right expansion reduced sequence utility pruning strategy]
	\label{Strategy:RERSU}
	\rm Given a ToSR $r$, according to the Theorem \ref{Theorem:RERSU}, when $r$ implements a right expansion with a specific item $i$ to generate a ToSR $r^\prime$, the utility of $r^\prime$ will not exceed the upper bound \textit{RERSU}, i.e., $u$($r^\prime$) $\le$ \textit{RERSU}($r$). If \textit{RERSU}($r^\prime$) $\textless$ \textit{minutil}, we can know that any ToSR extending from $r$ by performing a right expansion will not be a HTSR. Thus, we can stop extending further.
\end{strategy}

\begin{strategy}[Left expansion reduced sequence prefix extension utility pruning strategy]
	\label{Strategy:LERSPEU}
	\rm Given a ToSR $r$, according to Theorem \ref{Theorem:LERSPEU}, when $r$ implements a left expansion with a specific item $i$ to generate a ToSR $r^\prime$, the utility of $r^\prime$ will not exceed the upper bound \textit{LERSPEU}, i.e., $u$($r^\prime$) $\le$ \textit{LERSPEU}($r^\prime$). If \textit{LERSPEU}($r^\prime$) $\textless$ \textit{minutil}, we can know that any ToSR extending from $r$ by performing a left expansion will not be a HTSR. Thus, we can stop extending further.
\end{strategy}

\begin{strategy}[Right expansion reduced sequence prefix extension utility pruning strategy]
	\label{Strategy:RERSPEU}
	\rm Given a ToSR $r$, according to Theorem \ref{Theorem:RERSPEU}, when $r$ implements a right expansion with a specific item $i$ to generate a ToSR $r^\prime$, the utility of $r^\prime$ will not exceed the upper bound \textit{RERSPEU}, i.e., $u$($r^\prime$) $\le$ \textit{RERSPEU}($r^\prime$). If \textit{RERSPEU}($r^\prime$) $\textless$ \textit{minutil}, we can know that any ToSR extending from $r$ by performing a right expansion will not be a HTSR. Thus, we can stop to extend further.
\end{strategy}

Given a ToSR $r$ and the other ToSR $r^\prime$ generated from $r$ with an item $i$ by performing a left expansion, then we can know that the relationship of the upper bounds \textit{LEPEU}, \textit{LERSU}, and \textit{LERSPEU} is that \textit{LEPEU}($r^\prime$) $=$ \textit{LERSPEU}($r$) $\le$ \textit{LERSU}($r$) $\le$ \textit{LEPEU}($r$). Correspondingly, if $r^\prime$ generates from $r$ by performing a right expansion, then we can know that the relationship of the upper bounds \textit{REPEU}, \textit{RERSU}, and \textit{RERSPEU} is that \textit{REPEU}($r^\prime$) $=$ \textit{RERSPEU}($r$) $\le$ \textit{RERSU}($r$) $\le$ \textit{REPEU}($r$).

Since we stipulate that a left expansion will not follow a right expansion to avoid generating the same ToSR more times, this extension mode satisfies the anti-monotonic property. Note that this expansion style is different from HUSRM \cite{zida2015efficient}, US-Rule \cite{huang2021us}, and the other SRM algorithms like TRuleGrowth \cite{fournier2015mining} where they use a right-first expansion to avoid the repeated generation of a rule. Here we give a corresponding proof, shown below:

\begin{proof}
	\label{Proof:CAT}
	\rm Let $r$ $=$ $X$ $\rightarrow$ $Y$ be a ToSR. Since we stipulate that a left expansion will not follow a right expansion, the support value of the sequence $X$ is fixed when we perform a right expansion. In this case, the support of the entire ToSR $r$ still satisfies the anti-monotonic property, i.e., the support value of $r$ will remain constant or decrease. Thus, we can utilize this property to prune. However, if we use a right-first expansion, the support value of sequence $X$ varies, as does the support value of rule $r$. Therefore, the confidence value can get smaller, the same, or greater, i.e., it is unknown. That is why we begin with a left-first expansion. Besides, when we perform a left expansion, the confidence value of a ToSR is unknown. Therefore, we can only use the anti-monotonic property to prune when we perform the right expansion.
\end{proof}
 
\begin{strategy}[Confidence pruning strategy]
	\label{Strategy:CPS}
	\rm Given a ToSR $r$, when $r$ implements a right expansion with a specific item $i$ to generate the other ToSR $r^\prime$, the the confidence value of $r^\prime$ will be less or equal to $r$'s confidence value, i.e., \textit{conf}($r^\prime$) $\le$ \textit{conf}($r$). If \textit{conf}($r$) $\textless$ \textit{minconf}, we have \textit{conf}($r^\prime$) $\textless$ \textit{minconf} too. Thus, we can stop extending further.
\end{strategy}

\subsection{Data structures}

For mining rules effectively and efficiently, in HUSRM \cite{zida2015efficient} and US-Rule \cite{huang2021us}, they proposed a data structure called a utility table, which can maintain the necessary information about the candidate rules. In this paper, for the same purpose, we also proposed two homologous data structures called \textit{LE-utility table} and \textit{RE-utility table}. They can store the necessary information about the antecedent and consequent of a ToSR to compute the utility and confidence value effectively and efficiently. Besides, we also proposed a data structure called utility prefix sum list (\textit{UPSL}) for faster calculating \textit{LEPEU}, \textit{REPEU}, \textit{LERSPEU}, and \textit{RERSPEU} of a ToSR.

\begin{definition}[LE-element of a ToSR in a sequence and LE-utility table in $\mathcal{D}$]
\label{LE-element}
	\rm Given a ToSR $r$ and a sequence $s$ where sequence $s$ is the sequence that contains the antecedent of $r$, we use \textit{LEE}($r$, $s$) to represent the \textit{LE-element} of rule $r$ in sequence $s$. \textit{LEE}($r$, $s$) is defined as \textit{LEE}($r$, $s$) $=$ \textit{<SID, Uiltiy, LEPEU, REPEU, Positions, Indices>}. \textit{SID} is the sequence identifier of $s$ and $s$ should contain the antecedent of ToSR $r$. \textit{Utility} is the utility of $r$ in sequence $s$. \textit{LEPEU} means the left expansion prefix extension utility of $r$ in sequence $s$. \textit{REPEU} is the right expansion prefix extension utility of $r$ in sequence $s$. \textit{Positions} is a 3-tuple ($\alpha$, $\beta$, $\gamma$), in which $\alpha$ is the last item's position of antecedent of $r$, $\beta$ is the first item's position of consequent of $r$, and $\gamma$ is the last item's position of consequent of $r$. And \textit{Indices} is a 2-tuple ($\alpha^\prime$, $\gamma^\prime$), where $\alpha^\prime$ is the index of first item that can be extended into antecedent and $\gamma^\prime$ is the index of last item in sequence $s$. If the sequence only contains the antecedent of rule $r$, then we set \textit{Utility} $=$ 0, \textit{LEPEU} $=$ 0, \textit{REPEU} $=$ 0, \textit{Positions} $=$ ($\alpha$, $-1$, $-1$), \textit{Indices} $=$ ($-1$, $-1$) to represent that this sequence only for the support counting of antecedent of rule $r$. Given a ToSR $r$ and a sequence database $\mathcal{D}$, the \textit{LE-utility table} of rule $r$ in the database $\mathcal{D}$ is denoted as \textit{LEE}($r$) and defined as a table that consists of a set of \textit{LE-elements} of rule $r$.
\end{definition}

\begin{table}[!htbp]
	\centering
	\caption{\textit{LE-utility table} of HTSR $r_1$ in Table \ref{table1}}
	\label{table4}
	\begin{tabular}{|c|c|c|c|c|c|}  
		\hline 
		\textbf{SID} & \textbf{Utility} & \textbf{\textit{LEPEU}} & \textbf{\textit{REPEU}} & \textbf{Positions}& \textbf{Indices}\\
		\hline 
		\(s_{2}\) & 0 & 0 & 0 & ($4$, $-1$, $-1$) & ($-1$, $-1$)\\ 
		\hline 
		\(s_{3}\) & 17 & 0 & 0 & ($3$, $4$, $4$) & ($5$, $5$)\\
		\hline 
		\(s_{4}\) & 11 & 20 & 0 & ($2$, $4$, $4$) & ($4$, $6$)\\
		\hline 
	\end{tabular}
\end{table}

Note that we stipulate a left-first expansion, i.e., the right expansion will follow the left expansion. Thus, we should record the right expansion prefix extension utility for the right expansion. Furthermore, note that TotalSR is unlike HUSRM \cite{zida2015efficient} and US-Rule \cite{huang2021us} which can use a bit vector to compute efficiently the support of the antecedent of a rule as they do not care about the ordering of items in the antecedent. In TotalSR, we stipulate \textit{LE-element} records the sequence that contains the antecedent of rule $r$ to facilitate the calculation of the support of the antecedent of rule $r$. And the support of rule $r$ is equal to the number of \textit{LE-element} that \textit{Utility} $\neq$ 0.

For example, consider the HTSR $r_1$ $=$ $<$\{$e$, $f$\}, $c$$>$ $\rightarrow$ $<$$b$$>$ and sequence $s_4$. We have $u$($r_1$, $s_4$) $=$ 11, \textit{LEPEU}($r_1$, $s_4$) $=$ 20, \textit{REPEU}($r_1$, $s_4$) $=$ 0, \textit{Positions} $=$ ($2$, $4$, $4$), and \textit{Indices} $=$ ($4$, $6$). Thus, \textit{LEE}($r_1$, $s_4$) $=$ $<$$s_4$, $11$, $20$, $0$, ($2$, $4$, $4$), ($4$, $6$)$>$. Table \ref{table4} shows the \textit{LE-utility table} of rule $r_1$ as an example. Note that in sequence $s_2$, only antecedent occurs in this sequence. Thus \textit{LEE}($r_1$, $s_2$) $=$ $<$$s_2$, 0, 0, 0, (4, -1, -1), (-1, -1)$>$.

\begin{definition}[RE-element of a ToSR in a sequence and RE-utility table in $\mathcal{D}$]
\label{RE-element}
	\rm Given a ToSR $r$ and a sequence s where sequence $s$ is the sequence that contains the antecedent of $r$, we use \textit{REE}($r$, $s$) to denote the \textit{RE-element} of ToSR $r$ in sequence $s$. And \textit{REE}($r$, $s$) is defined as \textit{REE}($r$, $s$) $=$ $<$\textit{SID}, \textit{Uiltiy}, \textit{REPEU}, \textit{Position}, \textit{Index}$>$. \textit{SID} is the sequence identifier of $s$. \textit{Uiltiy} is the utility of $r$ in sequence $s$. \textit{REPEU} is the right expansion prefix extension utility of $r$ in sequence $s$. \textit{Position} is the last item's position of consequent of ToSR $r$. And last \textit{Index} is the index of last item in sequence $s$. If the sequence only contains the antecedent of rule $r$, then we set \textit{Uiltiy} $=$ 0, \textit{REPEU} $=$ 0, \textit{Position} $=$ $-1$, \textit{Index} $=$ $-1$ to represent that this sequence only for the support counting of antecedent of rule $r$. Given a ToSR $r$ and a sequence database $\mathcal{D}$, the \textit{RE-utility table} of ToSR $r$ in the database $\mathcal{D}$ is denoted as \textit{REE}($r$) and defined as a table that consists of a set of \textit{RE-elements} of ToSR $r$.
\end{definition}

\begin{table}[!htbp]
	\centering
	\caption{\textit{RE-utility table} of HTSR $r_3$ in Table \ref{table1}}
	\label{table5}
	\begin{tabular}{|c|c|c|c|c|}  
		\hline 
		\textbf{SID} & \textbf{Utility} & \textbf{\textit{REPEU}} & \textbf{Position}& \textbf{Index}\\
		\hline 
		\(s_{3}\) & 16 & 17 & 3 & 5\\
		\hline 
		\(s_{4}\) & 10 & 20 & 2 & 6\\
		\hline 
	\end{tabular}
\end{table}

Note that we stipulate a left-first expansion, i.e., the left expansion will not be conducted after the right expansion. Thus, we only record the \textit{REPEU} for the right expansion. Also, we stipulate that \textit{RE-element} records the sequence that contains the antecedent of ToSR $r$ to facilitate the calculation of the support of the antecedent of ToSR $r$. And the support of ToSR $r$ is equal to the number of \textit{RE-element} that \textit{Uiltiy} $\neq$ 0.

For example, consider the HTSR $r_3$ $=$ $<$\{$e$, $f$\}$>$ $\rightarrow$ $<$$c$$>$ and sequence $s_4$. We have $u$($r_3$, $s_4$) $=$ 10, \textit{REPEU}($r_3$, $s_4$) $=$ 20, \textit{Position} $=$ 2, and \textit{Index} $=$ $6$. Thus, \textit{REE}($r_3$, $s_4$) $=$ $<$$s_4$, 10, 20, 2, 2, 6$>$. Table \ref{table5} shows the \textit{RE-utility table} of HTSR $r_3$ as an example.

\begin{definition}[Utility prefix sum list in a sequence]
	\rm Given a sequence s with largest index equals to $k$, the utility prefix sum list of sequence $s$ is denoted as \textit{UPSL}($s$) and defined as \textit{UPSL}($s$) $=$ $<$$us_1$, $us_2$, $\cdots$, $us_k$$>$, where $us_i$ ($1$ $\le$ $i$ $\le$ $k$) is utility prefix sum of first $i$ items in sequence $s$.
\end{definition}

As an example, Table \ref{table6} shows the \textit{UPSL} of sequence $s_4$. We can calculate the \textit{LEPEU}($r_1$, $s_4$) $=$ $u$($r_1$, $s_4$) $+$ $us_5$ $-$ $us_3$ $=$ 11 $+$ 19 $-$ 10 $=$ 20. In general, we will scan $\mathcal{D}$ once after we remove all unpromising items from $\mathcal{D}$ to get the \textit{UPSL} of each sequence. With the help of \textit{UPSL} of a sequence, we can compute the \textit{LEPEU} and \textit{REPEU} in \textit{O}($1$) time, which in the past will cost \textit{O}($k$) and \textit{O}($l$) time, respectively, where $k$ and $l$ are the average number of the items that can extend to antecedent and consequent, respectively. Similarly, \textit{LERSPEU} and \textit{RERSPEU} can be computed in \textit{O}($1$) time, too.

\begin{table}[h]
	\centering
	\caption{The \textit{UPSL} of sequence $s_4$ in Table \ref{table1}}
	\label{table6}
	\begin{tabular}{|c|c|c|c|c|c|c|}  
		\hline 
		\textbf{item} & $e$ & $f$ & $c$ & $d$ & $g$ & $b$\\
		\hline 
		\textbf{index} & 1 & 2 & 3 & 4 & 5 & 6\\
		\hline 
		\textbf{$\rm us_{index}$} & 4 & 7 & 10 & 13 & 19 & 20\\
		\hline 
	\end{tabular}
\end{table}

\subsection{TotalSR algorithm}

\begin{algorithm}[h]
	\caption{TotalSR algorithm}
	\label{alg:TotalSR}
	\KwIn{$\mathcal{D}$: a sequence database, \textit{minconf}: the minimum confidence threshold, \textit{minutil}: the minimum utility threshold.}
	\KwOut{HTSRs: all high-utility totally-ordered sequential rules.}
	
	initialize $I$ $\leftarrow$ $\emptyset$;
	
	scan $\mathcal{D}$ to compute \textit{SEU}($i$) and $I$ $\cup$ \{$i$\};
	
	\While {\rm$\exists$ $i$ $\in$ $I$ and \textit{SEU}($i$) $\textless$ \textit{minutil}} {
		delete all unpromising items in $I$ and update \textit{SEU};
	}
	
	scan $\mathcal{D}$ to calculate \textit{UPSL}, generate $R$ (a set of ToSRs with size $=$ 1 $\ast$ 1), create \textit{LE-utility tables}, and compute the \textit{SEU} of rule $r$ $\in$ $R$; \\
	delete all unpromising ToSRs in $R$;
	
	\For {$r$ $\in$ $R$}{
		scan \textit{LEE}($r$) $\in$ \textit{LE-utility tables} to compute $u$($r$) and \textit{conf}($r$); \\
		\If {\rm$u$($r$) $\ge$ \textit{minutil} and \textit{conf}($r$) $\ge$ \textit{minconf}}{
			update HTSRs $\leftarrow$ HTSRs $\cup$ $r$;
		}
		\If {\rm\textit{LEPEU}($r$) $+$ \textit{REPEU}($r$) $-$ $u$($r$) $\ge$ \textit{minutil}}{
			call \textbf{leftExpansion}($r$, \textit{LEE}($r$));
		}
		\If{\rm\textit{conf}($r$) $\ge$ \textit{minconf} and \textit{REPEU}($r$) $\ge$ \textit{minutil}}{
		    call \textbf{rightExpansion}($r$, \textit{LEE}($r$), \textit{LEE}($r$).$length$);
		}
	}	
\end{algorithm}

\begin{algorithm}[h]
	\caption{The leftExpansion procedure}
	\label{alg:leftExpansion}
	\KwIn{$r$: a ToSR, \textit{LEE}($r$): \textit{LE-utility table} of $r$.}
	
	initialize \textit{LERSPEU} $\leftarrow$ $\emptyset$, \textit{LE-utility tables} $\leftarrow$ $\emptyset$, \textit{ToSRSet} $\leftarrow$ $\emptyset$;
	
	\For{$s_k$ $\in$ \textit{LEE}$(r)$}{
        \For{\rm$i$ $\in$ $s_k$ and $i$ can be extended into the antecedent of $r$}{
            $t$ $\leftarrow$ $i$ extended into the antecedent of $r$;   //both I- and S-expansion can be implemented\\
            \If{\rm$t$ is an illegal ToSR}{
                update \textit{LE-utility tables} of $t$;\\
            }
            \Else {
                update \textit{LERSPEU}($i$);\\
                \If {\textit{REPEU}($r$) $+$ \textit{LERSPEU}($i$) $\textless$ \textit{minutil}} {
					\If {$t$ $\in$ \textit{ToSRSet}} {
						delete $t$ from \textit{ToSRSet}; \\
						continue;  
					}
				}
				update \textit{ToSRSet} and \textit{LE-utility tables} of $t$;\\
            }
        }
	}
	
	\For {$t$ $\in$ \textit{ToSRSet}} {
		scan \textit{LEE}($t$) $\in$ \textit{LE-utility tables} to calculate $u$($t$) and \textit{conf}($t$); \\
		\If {\rm$u$($t$) $\ge$ \textit{minutil} and \textit{conf}($t$) $\ge$ \textit{minconf}} {
			update HTSRs $\leftarrow$ HTSRs $\cup$ $t$;
		}
		\If {\textit{LEPEU}$(t)$ $+$ \textit{REPEU}$(t)$ $-$ $u$$(t)$ $\ge$ \textit{minutil}}{
			call \textbf{leftExpansion}($t$, \textit{LEE}($t$));
		}
		\If{\rm\textit{conf}($t$) $\ge$ \textit{minconf} and \textit{REPEU}($t$) $\ge$ \textit{minutil}}{
		    call \textbf{rightExpansion}($t$, \textit{LEE}($t$), \textit{LEE}($r$).$length$);
		}
	}
\end{algorithm}

\begin{algorithm}[htbp]
	\caption{The rightExpansion procedure}
	\label{alg:rightExpansion} 
	\KwIn{$r$: a ToSR, \textit{REE}($r$): \textit{RE-utility table} of $r$, \textit{length}: the support value of antecedent of $r$.}
	
	initialize \textit{RERSPEU} $\leftarrow$ $\emptyset$, \textit{RE-utility tables} $\leftarrow$ $\emptyset$, \textit{ToSRSet} $\leftarrow$ $\emptyset$;
	
	\For{\rm$s_k$ $\in$ \textit{REE}($r$)}{
        \For{\rm$i$ $\in$ $s_k$ and $i$ can be extended into the consequent of $r$}{
            $t$ $\leftarrow$ $i$ extended into the consequent of $r$;   //both I- and S-expansion can be implemented\\
            update \textit{RERSPEU}($i$);\\
            \If {\rm\textit{RERSPEU}($i$) $\textless$ \textit{minutil}} {
				\If {$t$ $\in$ \textit{ToSRSet}} {
					delete $t$ from \textit{ToSRSet}; \\
					continue;  
				}
			}
			update \textit{ToSRSet} and \textit{RE-utility tables} of $t$;\\
        }
	}
	
	\For {$t$ $\in$ \textit{ToSRSet}} {
		scan \textit{REE}($t$) $\in$ \textit{RE-utility tables} to calculate $u$($t$) and \textit{conf}($t$); \\
		\If {\rm$u$($t$) $\ge$ \textit{minutil} and \textit{conf}($t$) $\ge$ \textit{minconf}} {
			update HTSRs $\leftarrow$ HTSRs $\cup$ $t$;
		}
		\If{\rm\textit{conf}($t$) $\ge$ \textit{minconf} and \textit{REPEU}($t$) $\ge$ \textit{minutil}}{
		    call \textbf{rightExpansion}($t$, \textit{REE}($t$), \textit{length});
		}
	}
	
\end{algorithm}

Based on the data structure and the pruning strategies mentioned above, the totally-ordered sequential rule mining algorithm, TotalSR, is proposed in this subsection. To avoid generating a rule twice and applying the confidence pruning strategy, we design a left-first expansion procedure. The main pseudocode of the TotalSR algorithm is shown in Algorithms \ref{alg:TotalSR}, \ref{alg:leftExpansion}, and \ref{alg:rightExpansion}.

TotalSR takes a sequence database $\mathcal{D}$, a minimum confidence threshold (\textit{minconf}), and a minimum utility threshold (\textit{minutil}) as its inputs and then outputs all high-utility totally-ordered sequential rules. TotalSR first scans $\mathcal{D}$ to compute the \textit{SEU} of all items and get the distinct items set $I$ (Lines 1-2). Afterward, all unpromising items will be deleted from $\mathcal{D}$ till all items are promising according to the UIP strategy (Lines 3-5). After that, TotalSR scans $\mathcal{D}$ again to calculate \textit{UPSL} for each sequence, generates all 1 $\ast$ 1 ToSRs, creates the responding \textit{LE-utility table}, and computes the \textit{SEU} of all rules. And then TotalSR will delete all unpromising ToSRs (Lines 6-7). Next, for each 1 $\ast$ 1 candidate rule $r$, TotalSR scans \textit{LEE}($r$) to compute $u$($r$) and \textit{conf}($r$) (Line 9). If ToSR $r$ satisfies \textit{minutil} and \textit{minconf}, TotalSR will update the HTSRs set with $r$ (Lines 10-12). If \textit{LEPEU}($r$) $+$ \textit{REPEU}($r$) $-$ \textit{u}($r$) greater than or equal to \textit{minutil}, TotalSR will implement leftExpansion. Note that TotalSR will perform right expansion after left expansion. Thus, it is necessary to add \textit{REPEU}($r$). And because both \textit{LEPEU}($r$) and \textit{REPEU}($r$) include $u$($r$) we should subtract $u$($r$) one time to get the correct utility upper bound (Lines 13-15). Subsequently, according to confidence pruning strategy if \textit{conf}($r$) $\ge$ \textit{minconf} and \textit{REPEU}($r$) $\ge$ \textit{minutil}, TotalSR will call rightExpansion (Lines 16-18).

In Algorithm \ref{alg:leftExpansion}, the TotalSR will perform leftExpansion. It takes a ToSR $r$ and a \textit{LE-utility table} \textit{LEE}($r$) of ToSR $r$ as the inputs. Firstly, Algorithm \ref{alg:leftExpansion} initializes \textit{LERSPEU}, \textit{LE-utility tables}, and the ToSRs set \textit{ToSRSet} to the empty set (Line 1). Then for each sequence $s_k$ in \textit{LEE}($r$) and each item $i$ in $s_k$ where $i$ can be extended into the antecedent of $r$, the item $i$ will be extended into the antecedent of $r$ to form a temporary rule $t$ (Lines 2-4). If $t$ is an illegal ToSR the leftExpansion will only update \textit{LE-utility tables} of $t$, otherwise it will update \textit{LERSPEU} of $i$. And then if the \textit{REPEU}($r$) $+$ \textit{LERSPEU}($i$) $\textless$ \textit{minutil} and \textit{t} $\in$ \textit{ToSRSet}, leftExpansion will delete \textit{t} from \textit{ToSRSet} and continue. After that leftExpansion will update \textit{ToSRSet} and \textit{LE-utility tables} of \textit{t} (Lines 5-16). Note that leftExpansion will still perform right expansions after left expansions. Thus it should plus the value of \textit{REPEU}($r$) to determine whether a ToSR should be pruned. If \textit{t} cannot form a legal ToSR, it only updates \textit{LE-utility tables} of $i$ to count the support value of a ToSR with the antecedent being extended an item $i$. Lastly, for each ToSR \textit{t} $\in$ \textit{ToSRSet}, leftExpansion will scan \textit{LEE}($t$) to compute $u$($t$) and \textit{conf}($t$) (Line 21). If both $u$($t$) and \textit{conf}($t$) exceed the thresholds, it will update HTSRs set with $t$ (Lines 22-24). Then if \textit{LEPEU}($t$) $+$ \textit{REPEU}($t$) $-$ \textit{u}($t$) $\ge$ \textit{minutil} leftExpansion will call leftExpansion once to further expand the antecedent (Lines 25-27). According to the confidence pruning strategy if \textit{conf}($t$) $\ge$ \textit{minconf} and \textit{REPEU}($t$) $\ge$ \textit{minutil}, leftExpansion will call a rightExpansion (Lines 28-30).

In Algorithm \ref{alg:rightExpansion}, TotalSR will perform rightExpansion. It is very similar to Algorithm \ref{alg:leftExpansion}. Therefore, we will not describe the details of rightExpansion. However, there are still two differences compared to Algorithm \ref{alg:leftExpansion}. One is that when TotalSR implements rightExpansion the antecedent of the input ToSR is fixed, i.e., the support value of the antecedent is determined. Moreover, the rightExpansion will definitely form a legal rule. Therefore, rightExpansion takes the support value of the antecedent of a ToSR as the input to fast compute the confidence (Lines 3-13). The other is that the rightExpansion will not perform a leftExpansion since it stipulates a left-first expansion (Lines 15-23).

\subsection{TotalSR$^+$ algorithm}

In TotalSR, the data structures \textit{LE-utility table} and \textit{RE-utility table} will record the sequences that the antecedent of a ToSR occurs for the convenience of computation of the support value of the antecedent. Thus, there will be a lot of entries in the utility table that record sequences that do not make a contribution to utility computation. That will cost a lot of time when we update the utility table. Therefore, we propose a novel algorithm, TotalSR$^+$, which is more efficient than TotalSR in terms of execution time and scalability. To achieve that, we redesign the \textit{LE-utility table} and \textit{RE-utility table} and introduce an auxiliary antecedent record table (\textit{ART}) to count the sequences that cannot form a legal ToSR.

\begin{definition}[\textit{LE-element}$^+$ of a ToSR in a sequence and \textit{LE-utility}$^+$ table in $\mathcal{D}$]
	\rm Given a ToSR $r$ and a sequence $s$ where $s$ is the sequence that contains $r$, the \textit{LE-element}$^+$ of rule $r$ in sequence $s$ is denoted as \textit{LEE}$^+$($r$, $s$) and defined as \textit{LEE}$^+$($r$, $s$) $=$ $<$\textit{SID}, \textit{Utility}, \textit{LEPEU}, \textit{LEPEU}, \textit{Positions}, \textit{Indices}$>$. \textit{SID}, \textit{Utility} \textit{LEPEU}, \textit{REPEU}, \textit{Positions}, and \textit{Indices} are the same as definition \ref{LE-element}, while \textit{LE-element}$^+$ only records the sequences that contain a legal ToSR. Given a ToSR $r$ and a sequence database $\mathcal{D}$, the \textit{LE-utility}$^+$ \textit{table} of rule $r$ in the database $\mathcal{D}$ is denoted as \textit{LEE}$^+$($r$) and defined as a table that consists of a set of \textit{LE-element}$^+$ of rule $r$.
\end{definition}

\begin{definition}[RE-element$^+$ of a ToSR in a sequence and \textit{RE-utility}$^+$ \textit{table} in $\mathcal{D}$]
	\rm Given a ToSR $r$ and a sequence $s$ where $s$ is the sequence that contains $r$, the \textit{RE-element}$^+$ of rule $r$ in sequence $s$ is denoted as \textit{REE}$^+$($r$, $s$) and defined as \textit{REE}$^+$($r$, $s$) $=$ $<$\textit{SID}, \textit{Uiltiy}, \textit{REPEU}, \textit{Position}, \textit{Index}$>$. \textit{Uiltiy}, \textit{REPEU}, \textit{Position}, \textit{Index} are the same as definition \ref{RE-element}, while \textit{RE-element}$^+$ only records the sequences that contain a legal ToSR. Given a ToSR $r$ and a sequence database $\mathcal{D}$, the \textit{RE-utility}$^+$ \textit{table} of rule $r$ in the database $\mathcal{D}$ is denoted as \textit{REE}$^+$($r$) and defined as a table that consists of a set of \textit{RE-element}$^+$ of rule $r$.
\end{definition}

\begin{definition}[Auxiliary antecedent record table]
	\rm Given a ToSR $r$ $=$ $X$ $\rightarrow$ $Y$ and a sequence database $\mathcal{D}$, the auxiliary antecedent record table of $r$ is defined as a \textit{ART}($r$) $=$ \{\textit{key}: \textit{value}\}, where key is the antecedent of $r$, i.e., $X$, and value is the set of sequences that contain the antecedent of $r$ in $\mathcal{D}$ but cannot form the legal ToSR. Note that the auxiliary antecedent record table only records the sequences that cannot form a legal ToSR.
\end{definition}

\begin{table}[h]
	\centering
	\caption{\textit{LE-utility}$^+$ \textit{table} of $r_1$ in Table \ref{table1}}
	\label{table7}
	
	\begin{tabular}{|c|c|c|c|c|c|}  
		\hline 
		\textbf{SID} & \textbf{Utility} & \textbf{\textit{LEPEU}} & \textbf{\textit{REPEU}} & \textbf{Positions}& \textbf{Indices}\\
		\hline 
		\(s_{3}\) & 17 & 0 & 0 & ($3$, $4$, $4$) & ($5$, $5$)\\
		\hline 
		\(s_{4}\) & 11 & 20 & 0 & ($2$, $4$, $4$) & ($4$, $6$)\\
		\hline 
	\end{tabular}
\end{table}

For example, consider the HTSR $r_1$ $=$ $<$\{$e$, $f$\}, $c$$>$ $\rightarrow$ $<$$b$$>$ and sequence $s_4$, in Table \ref{table3}. Its \textit{LE-utility table} contains three sequences shown in table \ref{table4}. However, in TotalSR$^+$ its \textit{LE-utility}$^+$ \textit{table} only contains two sequences shown in Table \ref{table7}, the sequence $s_2$ that cannot form a legal ToSR is recorded in \textit{ART}($r_1$) $=$ \{\{$e$, $f$\}: \{$s_2$\}\}. 

Compared to TotalSR, with the help of the auxiliary antecedent record table, TotalSR$^+$ can easily calculate the support value of the antecedent of a ToSR $r$. It is equal to the size of \textit{LE-utility}$^+$ \textit{table} plus the size of the value of \textit{ART}($r$). Based on this optimization, TotalSR$^+$ can tremendously reduce the execution time.

\begin{algorithm}[h]
	\caption{TotalSR$^+$ algorithm}
	\label{alg:TotalSR^+}
	\KwIn{$\mathcal{D}$: a sequence database, \textit{minutil}: the utility threshold, \textit{minconf}: the confidence threshold.}
	\KwOut{HTSRs: all high-utility totally-ordered sequential rules.}
	
	initialize $I$ $\leftarrow$ $\emptyset$;
	
	scan $\mathcal{D}$ to compute \textit{SEU}($i$) and $I$ $\cup$ \{$i$\};
	
	\While {\rm$\exists$ $i$ $\in$ $I$ and $\textit{SEU}(i$) $\textless$ \textit{minutil}} {
		delete all unpromising items in $I$ and	update the \textit{SEU}; 
	}
	
	scan $\mathcal{D}$ to calculate \textit{UPSL}, generate $R$ (a set of ToSRs with size $=$ 1 $\ast$ 1), create \textit{LE-utility}$^+$ \textit{tables} and the corresponding \textit{ARTs}, and compute the \textit{SEU} of rule $r$ $\in$ $R$; \\
	delete all unpromising ToSRs in $R$;
	
	\For {$r$ $\in$ $R$}{
		scan \textit{LEE}$^+$($r$) $\in$ \textit{LE-utility}$^+$ \textit{tables} and \textit{ART}($r$) $\in$ \textit{ARTs} to compute $u$($r$) and \textit{conf}($r$); \\
		\If {\rm$u$($r$) $\ge$ \textit{minutil} and \textit{conf}($r$) $\ge$ \textit{minconf}}{
			update HTSRs $\leftarrow$ HTSRs $\cup$ $t$;
		}
		\If {\rm\textit{LEPEU}($r$) $+$ \textit{REPEU}($r$) $-$ $u$($r$) $\ge$ \textit{minutil}}{
			call \textbf{leftExpansion}$^+$($r$, \textit{LEE}$^+$($r$), \textit{ART}($r$));
		}
		\If{\rm\textit{conf}($r$) $\ge$ \textit{minconf} and \textit{REPEU}($r$) $\ge$ \textit{minutil}}{
    			call \textbf{rightExpansion}($r$, \textit{LEE}$^+$($r$), \textit{ART}($t$).$length$ $+$ \textit{LEE}$^+$($t$).$length$);
		}
	}	
\end{algorithm}

\begin{algorithm}[h]
	\caption{The leftExpansion$^+$ procedure}
	\label{alg:leftExpansion^+}
	
	\KwIn{$r$: a ToSR, \textit{LEE}$^+$($r$): \textit{LE-utility}$^+$ \textit{table} of $r$, \textit{ART}($r$): the auxiliary antecedent record table of $r$.}
	
	initialize \textit{LERSPEU} $\leftarrow$ $\emptyset$, \textit{LE-utility}$^+$ \textit{tables} $\leftarrow$ $\emptyset$, \textit{ToSRSet} $\leftarrow$ $\emptyset$;
	
	\For{\rm$s_k$ $\in$ \textit{LEE}$^+$($r$)}{
        \For{\rm$i$ $\in$ $s_k$ and $i$ can be extended into the antecedent of $r$}{
            $t$ $\leftarrow$ $i$ extended into the antecedent of $r$;   //both I- and S-expansion can be implemented\\
            \If{\rm$t$ is an illegal ToSR}{
                update \textit{ART}($t$);\\
            }
            \Else {
                update \textit{LERSPEU}($i$);\\
                \If {\rm\textit{REPEU}($r$) $+$ \textit{LERSPEU}($i$) $\textless$ \textit{minutil}} {
					\If {$t$ $\in$ \textit{ToSRSet}} {
						delete $t$ from \textit{ToSRSet}; \\
						continue;  
					}
				}
				update \textit{ToSRSet} and \textit{LE-utility tables} of $t$;\\
            }
        }
	}
	
	\For {$t$ $\in$ \textit{ToSRSet}} {
	    update \textit{ART}($t$);\\
		scan \textit{LEE}$^+$($t$) $\in$ \textit{LE-utility}$^+$ \textit{tables} and \textit{ART}($t$) to calculate $u$($t$) and \textit{conf}($t$); \\
		\If {\rm$u$($t$) $\ge$ \textit{minutil} and \textit{conf}($t$) $\ge$ \textit{minconf}} {
			update HTSRs $\leftarrow$ HTSRs $\cup$ $t$;
		}
		\If {\textit{LEPEU}($t$) $+$ \textit{REPEU}($t$) $-$ $u$($t$) $\ge$ \textit{minutil}}{
			call \textbf{leftExpansion}$^+$($t$, \textit{LEE}$^+$($t$), \textit{ART}($t$));
		}
		\If{\rm\textit{conf}($t$) $\ge$ \textit{minconf} and \textit{REPEU}($t$) $\ge$ \textit{minutil}}{
    			call \textbf{rightExpansion}($t$, \textit{LEE}$^+$($t$), \textit{ART}($t$).$length$ $+$ \textit{LEE}$^+$($t$).$length$);
		}
	}
\end{algorithm}

In TotalSR$^+$, the Algorithms \ref{alg:TotalSR^+} and \ref{alg:leftExpansion^+} are very similar to Algorithms \ref{alg:TotalSR} and \ref{alg:leftExpansion}, respectively. Therefore, we will not describe the details again in Algorithms \ref{alg:TotalSR^+} and \ref{alg:leftExpansion^+}. On the contrary, we just mention the different parts. As for the right expansion in TotalSR$^+$, we use the same procedure of rightExpansion in TotalSR. Because when TotalSR$^+$ implements a right expansion, the only difference is that TotalSR$^+$ will use \textit{REE}$^+$($t$) rather than \textit{REE}($t$).

Compared to Algorithm \ref{alg:TotalSR}, in Algorithm \ref{alg:TotalSR^+} TotalSR$^+$ needs to create the auxiliary antecedent record tables of each 1 $\ast$ 1 rule for the convenience of computation of antecedent's support value (Line 6). To compute the confidence value of the ToSR $r$, TotalSR$^+$ needs to combine the \textit{LEE}$^+$($r$) and \textit{ART}($r$) (Line 9). Note that the utility value can be computed using only \textit{LEE}$^+$($r$). When the leftExpansion$^+$ is called, TotalSR$^+$ will pass \textit{ART}($r$) to guarantee the correctness of the antecedent's support value of the next extended ToSR (Line 14). When TotalSR$^+$ calls rightExpansion, the support of the antecedent is fixed. Thus, TotalSR$^+$ just needs to pass the length of the value in the auxiliary antecedent record table of $r$ to ensure the correctness of the antecedent's support value calculation of the next extended ToSR (Line 17).

In Algorithm \ref{alg:leftExpansion^+}, TotalSR$^+$ takes an extra auxiliary antecedent record table of $r$ to ensure the correctness of the support of antecedent as its input compared to Algorithm \ref{alg:leftExpansion}. When TotalSR$^+$ forms an illegal ToSR, TotalSR$^+$ will update \textit{ART}($t$) instead of \textit{LE-utility}$^+$ \textit{tables} of $t$, since in TotalSR$^+$ \textit{LE-utility}$^+$ \textit{tables} only records the legal ToSR (Line 6). Besides, for those antecedents that cannot form legal ToSR in \textit{ART}($r$), TotalSR$^+$ also needs to update them to ensure the correctness of the support of antecedents (Lines 21). Then TotalSR$^+$ scans \textit{LEE}$^+$($t$) and \textit{ART}($t$) to calculate $u$($t$) and \textit{conf}($t$) (Line 22). Similar to Algorithm \ref{alg:TotalSR^+}, when TotalSR$^+$ performs leftExpansion$^+$ it will also pass \textit{ART}($r$) to guarantee the correctness of the antecedent's support value of the next extended ToSR (Line 27). When TotalSR$^+$ implements rightExpansion, it will only pass the length of the value in the auxiliary antecedent record table of $t$ (Line 30).

\section{Experiments}  \label{sec:exps}

In this section, we designed experiments to verify the performance of the two proposed algorithms, TotalSR and TotalSR$^+$. Since there is no existing work about mining totally-ordered sequential rules, we used TotalSR to conduct ablation studies to demonstrate the effect of the pruning strategies and to prove the effectiveness of each pruning strategy simultaneously. To do this, we designed six algorithms for TotalSR with different optimization and pruning strategies. These algorithms are denoted as TotalSR$\rm_{Bald}$, TotalSR$\rm_{SEU}$, TotalSR$\rm_{SEU^-}$, TotalSR$\rm_{RSU}$, TotalSR$\rm_{RSPEU}$, and TotalSR, respectively. TotalSR$\rm_{Bald}$ does not use any pruning strategy except the \textit{UPSL} optimization. TotalSR$\rm_{SEU}$ uses \textit{SEU} and \textit{UPSL}, while TotalSR$\rm_{SEU^-}$ only uses \textit{SEU} to verify the efficiency of \textit{UPSL}. Furthermore, \textit{LEPEU}, \textit{REPEU}, \textit{LERSU}, and \textit{RERSU} are used in TotalSR$\rm_{RSU}$. To validate the validity of the two novel pruning strategies proposed in this paper, TotalSR$\rm_{RSPEU}$ uses pruning strategies \textit{LEPEU}, \textit{REPEU}, \textit{LERSPEU}, and \textit{RERSPEU}. Naturally, the sixth variant, TotalSR, uses all the pruning strategies, including the confidence pruning strategy. In addition, TotalSR$^+$ adopts all pruning strategies and utilizes \textit{ART} to further improve the performance. Except for TotalSR$\rm_{SEU^-}$, which is designed to evaluate the effectiveness of \textit{UPSL}, all the other algorithms utilize \textit{UPSL} to perform optimization.

All algorithms are implemented in Java with JDK 11.0.15, and the machine used for all experiments has a 3.8 GHz Intel Core i7-10700K processor, 32 GB of RAM, and a 64-bit version of Windows 10. All experimental results are listed below. Note that if the corresponding algorithm takes more than 10,000 seconds to execute, the mining process will be stopped.

\subsection{Data description}

The performance of the two proposed algorithms is assessed using six datasets, comprising four real-life datasets and two synthetic datasets. The four real-life datasets, including Leviathan, Bible, Sign, and Kosarak10k are generated from a book, a book, sign language, and click-stream, respectively. All real-life datasets can be downloaded from the open-source website SPMF \cite{fournier2016spmf}. SynDataset-10k and SynDataset-20k are two synthetic datasets that are generated from the IBM data generator \cite{agrawal1995mining}. Moreover, we use the same simulation model that was universally used in \cite{tseng2012efficient, liu2012mining, lin2017fdhup, gan2020proum} to generate the internal utility and the external utility for each dataset. The descriptions of these datasets are listed in Table \ref{datasets}. $\vert \mathcal{D} \vert$ represents the number of sequences in each dataset. The distinct quantity of items in each dataset is denoted as $\vert \textit{I} \vert$. The average number and the maximum number of itemsets in each dataset are denoted as \textit{avg(S)} and \textit{max(S)}, respectively. The average length of the sequence in each dataset is expressed as \textit{avg(Seq)}. And \textit{avg(Ele)} means the average items in each itemset of the sequence in every dataset.

\begin{table}[h]
	\caption{Description of the six datasets}  
	\label{datasets}
	\centering
	\begin{tabular}{|c|c|c|c|c|c|c|}
    	\hline
    	\textbf{Dataset} & \textbf{$\vert \mathcal{D} \vert$} & \textbf{$\vert \textit{I} \vert$} & \textbf{$\textit{avg}(\textit{S})$} & \textbf{$\textit{max}(\textit{S})$} & \textbf{$\textit{avg}(\textit{Seq})$} & \textbf{\textit{avg}(\textit{Ele})}   \\
    	\hline 
            Bible & 36,369 & 13,905 & 21.64 & 100 & 21.64 & 1.00  \\ \hline
            Kosarak10k & 10,000 & 10,094 & 8.14 & 608 & 8.14 & 1.00  \\ \hline
            Leviathan & 5,834 & 9,025 & 33.81 & 100 & 33.81 & 1.00  \\ \hline
            SIGN & 730 & 267 & 52.00 & 94 & 52.00 & 1.00  \\ \hline
    		SynDataset10k  & 10,000  & 7312 & 6.22 & 18 & 26.99 & 4.35 \\ \hline
            SynDataset20k  & 20,000  & 7442 & 6.20 & 18 & 26.97 & 4.35 \\ 
        \hline 
	\end{tabular}
\end{table}

\subsection{Efficiency analysis}

In this subsection, we will analyze the efficiency of the proposed two algorithms, TotalSR and TotalSR$^+$, in terms of run time. To keep track of the results that we want to compare with the different version algorithms in the last subsection we designed, we set the minimum confidence threshold to 0.6, but the minimum utility threshold will be varied according to the characteristics of the different datasets. All experimental results are illustrated in Fig. \ref{runtime}.

\begin{figure*}[h]
	\centering
	\includegraphics[width=1\linewidth]{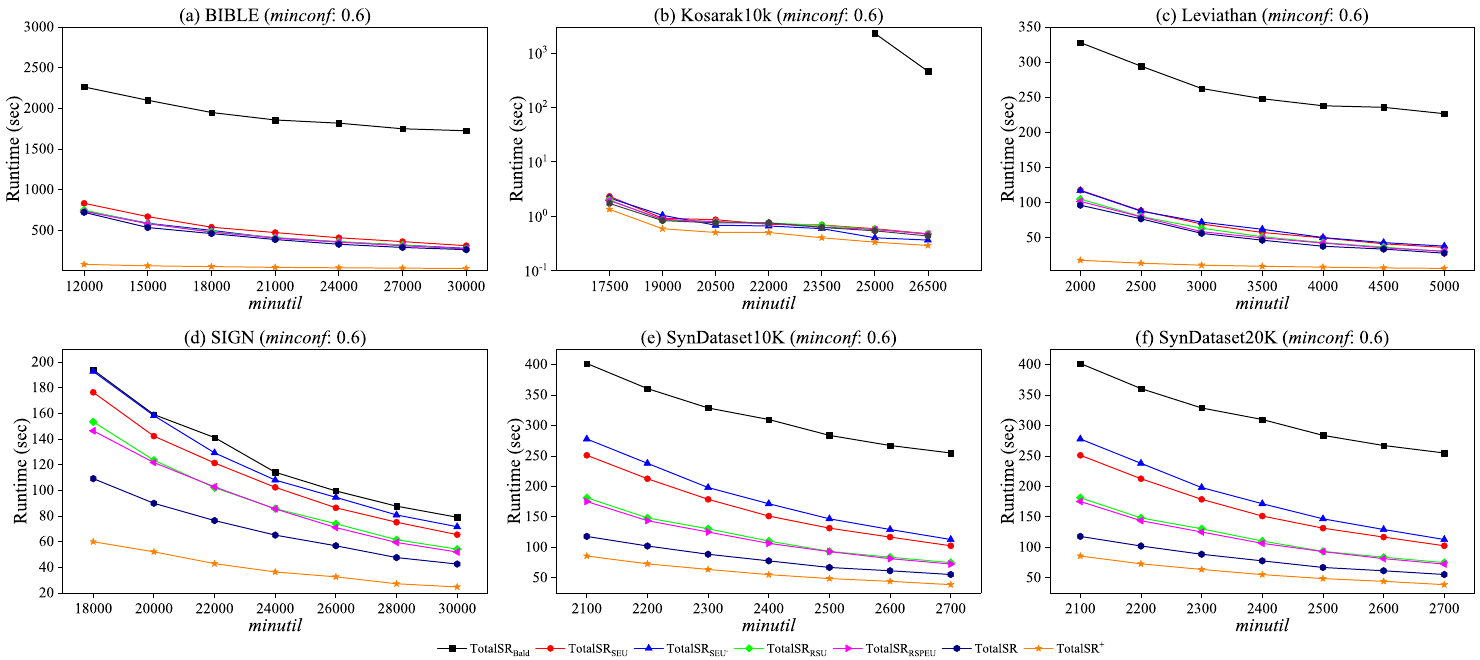}
	\caption{Execution time results under various minimum utility thresholds}
	\label{runtime}
\end{figure*}

From Fig. \ref{runtime}, we can find that on all datasets and under any minimum utility threshold, TotalSR$^+$ can achieve the best performance compared to all the other algorithms, demonstrating that TotalSR$^+$ is the best method in terms of execution time. Also, the results of TotalSR$\rm_{Bald}$ show that the algorithm without any pruning strategy will take an extremely and unacceptable long time to get the results that other methods can easily get. Total$\rm_{Bald}$ cannot produce results in 1000 seconds on Kosarak10k, whereas the other algorithms can produce results in a few seconds. Note that the curve of Total$\rm_{Bald}$ when \textit{minutil} less than or equal to 23500 is not drawn in Fig. \ref{runtime}(b) since it can not finish the mining process within 10,000 seconds. On datasets that consist of longer sequences, the data structure of \textit{UPSL} can realize great efficiency promotion, such as on the Leviathan, SIGN, SynDataset10k, and SynDataset20k. However, the performance of TotalSR$\rm_{SEU}$ on BIBLE and Kosarak10k show that \textit{UPSL}-based variant is relatively weaker than TotalSR$\rm_{SEU^-}$, especially on \textit{BIBLE}. This is because \textit{UPSL} must be initialized first in TotalSR$\rm_{SEU}$, which takes $O(nm)$ more time, where $n$ is the number of sequences in the dataset and $m$ is the average length of the sequences. While in TotalSR$\rm_{SEU^-}$, we do not waste time initializing \textit{UPSL}. From the curves of the execution time of TotalSR$\rm_{RSU}$ and TotalSR$\rm_{RSPEU}$ shown in Fig. \ref{runtime}, we can figure out that TotalSR$\rm_{RSU}$, which uses the traditional upper bounds, is obviously slower than TotalSR$\rm_{RSPEU}$ on all datasets and under any minimum utility threshold, which proves that the two novel upper bounds are effective and can contribute to better performance. Moreover, the results of TotalSR on each dataset show that the confidence pruning strategy can realize extraordinary effects. In summary, the abundant results demonstrate that TotalSR using the two novel pruning strategies can discover all HTSRs in an acceptable time. Furthermore, TotalSR$^+$, which uses the auxiliary optimization, can realize better performance and be suitable for all datasets and any conditions.

\subsection{Effectiveness of the pruning strategies}

In this subsection, to validate the effectiveness of the two novel pruning strategies, we compare five algorithms, including TotalSR$\rm_{Bald}$, TotalSR$\rm_{SEU}$, TotalSR$\rm_{RSU}$, TotalSR$\rm_{RSPEU}$, and TotalSR. They can demonstrate that \textit{LERSPEU}, \textit{RERSPEU}, and the confidence pruning strategy are more effective by comparing the generated number of candidates and the number of HTSRs with different minimum utility thresholds and different datasets. The results are shown in Fig. \ref{effectiveness}. Note that the left Y axis is the number of candidate HTSRs using the histogram to display and the right Y axis is the number of real HTSRs using a red curve to represent them.

As Fig. \ref{effectiveness} shown with the \textit{minutil} increases, the number of corresponding candidates decreases, while the extent of the decrease of different methods is distinct. TotalSR$\rm_{Bald}$ has the highest number of candidates across all datasets since it does not use any pruning strategy. With the help of the \textit{SEU} pruning strategy, the number of candidates for TotalSR$\rm_{SEU}$ decreases sharply on four real-life datasets. The effect is not very obvious but still works on two synthetic datasets. Compared to the \textit{SEU} pruning strategy, the pruning strategies of \textit{LERSU} and \textit{RERSU} further sharply reduce the number of candidates on all datasets, even the two synthetic datasets. Thus, they are able to reduce the number of candidates. Moreover, the number of candidates generated by the two novel pruning strategies of \textit{LERSPEU} and \textit{RERSPEU} on all datasets is smaller than the traditional pruning strategies based on \textit{LERSU} and \textit{RERSU}. Even if the reduction in the number of candidates is not particularly noticeable, it still has an effect. As shown in Fig. \ref{runtime}, TotalSR$\rm_{RSPEU}$ takes less time than TotalSR$\rm_{RSU}$ to get the same results. The other important point is that the quantity of candidates produced by the confidence pruning strategy drops further on each dataset. The results shown in Fig. \ref{effectiveness} demonstrate that the \textit{LERSPEU}, \textit{RERSPEU}, and the confidence pruning strategy are extraordinarily effective.

\begin{figure*}[h]
	\centering
	\includegraphics[width=1\linewidth]{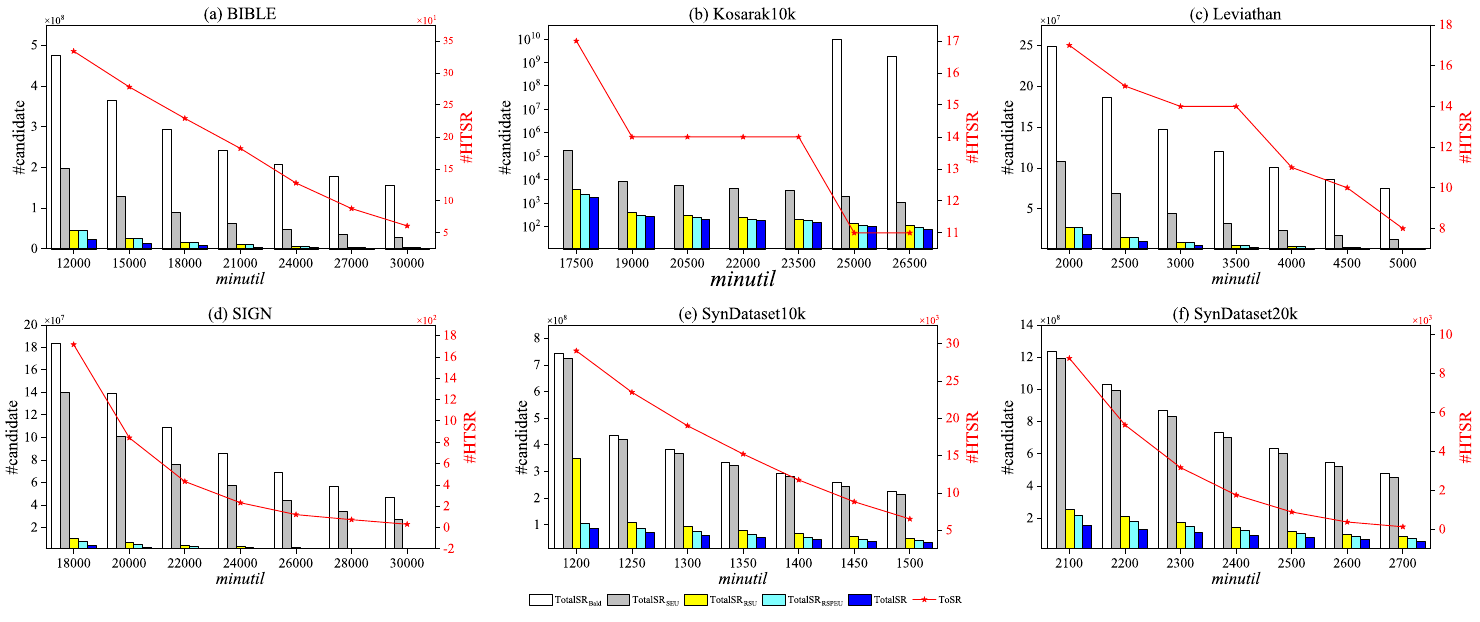}
	\caption{The generated candidates and HTSRs under different pruning strategies and various minimum utility thresholds}
	\label{effectiveness}
\end{figure*}

\begin{figure*}[h]
	\centering
	\includegraphics[width=1\linewidth]{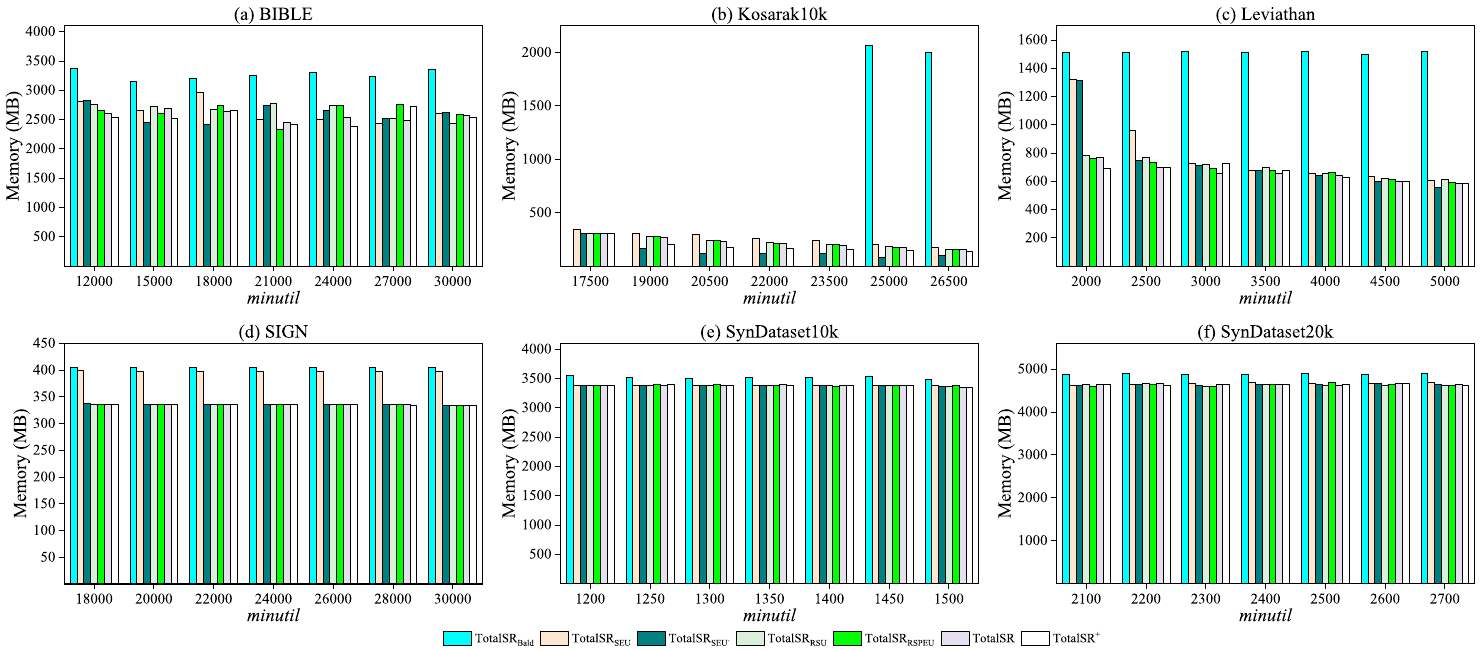}
	\caption{The memory usage of the proposed methods under various minimum utility thresholds}
	\label{memory}
\end{figure*}

\begin{figure*}[h]
	\centering
	\includegraphics[width=1\linewidth]{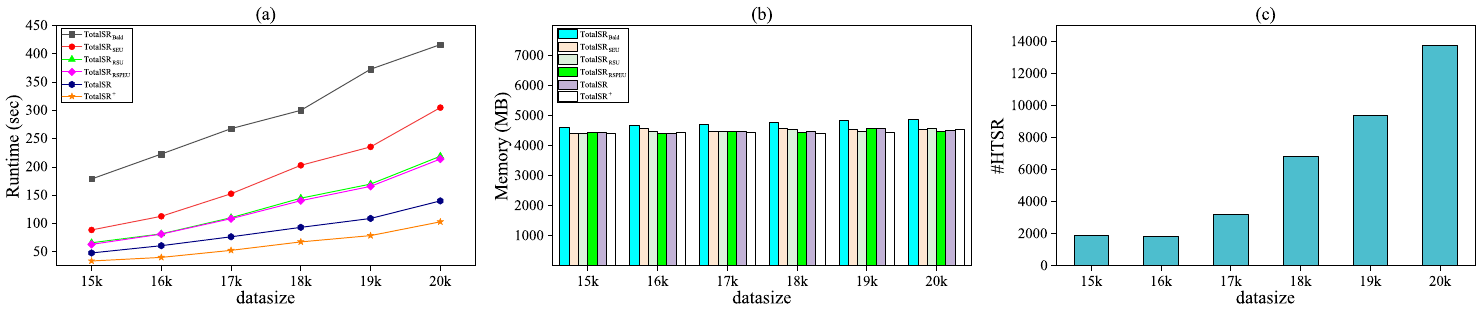}
	\caption{The scalability test of the proposed algorithms}
	\label{scalability}
\end{figure*}

\subsection{Memory evaluation}

In addition to analyzing the running time, memory consumption is also a crucial measure of the effectiveness of the algorithm. In this subsection, we access the memory consumption of the different versions of the algorithms. The details are presented in Fig. \ref{memory}.

On the BIBLE dataset, TotalSR$\rm_{Bald}$ costs the most memory, while the memory consumption of other algorithms is different. Under the small \textit{minutil}, TotalSR$\rm_{RSU}$ and TotalSR$\rm_{RSPEU}$ cost less memory relatively, but in large \textit{minutil} TotalSR$\rm_{SEU}$ and TotalSR$\rm_{SEU^-}$ consumed less memory. On Kosarak10k dataset, the memory cost of TotalSR$\rm_{Bald}$ was severely high. On the contrary, the other algorithms took up much less memory. TotalSR$\rm_{SEU^-}$ used less memory than the remaining algorithms, because the other algorithms needed to keep the \textit{UPSL} in memory. On Leviathan, TotalSR$\rm_{Bald}$ cost most memory as usual, and TotalSR$\rm_{SEU}$ consumed a little less than TotalSR$\rm_{Bald}$, while other algorithms used relatively less memory, which is similar to BIBLE. On the dataset SIGN, TotalSR$\rm_{Bald}$ and TotalSR$\rm_{ SEU}$ consumed high memory, while TotalSR$\rm_{SEU^-}$, TotalSR$\rm_{RSU}$, TotalSR$\rm_{RSPEU}$, TotalSR, and TotalSR$^+$ used small amount of memory and almost to be the same. On SynDataset10k and SynDataset20k, TotalSR$\rm_{Bald}$ still costs the most memory, but the other methods occupied almost the same memory. With the help of the pruning strategies, the memory consumption of the algorithm will be greatly reduced compared to the method that do not use all the pruning strategies since they can remove the unpromising items to save a lot of memory. Meanwhile, although the \textit{UPSL} will occupy some memory, it is just a small part, and the advantage of \textit{UPSL} is much better.

\subsection{Scalability test}

To verify the robustness of the proposed algorithms, we conducted several experiments on the dataset \cite{agrawal1995mining} with sizes ranging from 15k to 20k, and the minimum utility threshold was fixed at 2000. Moreover, TotalSR$\rm_{SEU^-}$ will not be tested since TotalSR$\rm_{SEU^-}$ is the same as TotalSR$\rm_{SEU}$. All results are presented in Fig. \ref{scalability}. As the dataset size increases, the execution time of each algorithm will increase simultaneously. TotalSR$^+$ still takes the shortest time, followed by TotalSR, TotalSR$\rm_{RSPEU}$, TotalSR$\rm_{RSU}$, TotalSR$\rm_{SEU}$, and TotalSR$\rm_{Bald}$. This is similar to the results in Fig. \ref{runtime}. When it comes to memory consumption, Fig. \ref{scalability}(b) shows that TotalSR$\rm_{Bald}$ costs the most memory, and the memory used by other algorithms is similar to Fig. \ref{memory}(e) and Fig. \ref{memory}(f). As the results shown in Fig. \ref{scalability}(c), the number of HTSRs will gradually increase with the size of dataset. And it is in line with our expectation. The scalability test results show that the proposed algorithms have excellent scalability, and TotalSR$^+$ performs unquestionably best under all conditions.
\section{Conclusion} \label{sec:concs}

In this work, we formulated the problem of totally-ordered sequential rule mining and proposed the basic algorithm TotalSR to discover the complete HTSRs in a given dataset. To verify the effectiveness of the pruning strategies proposed in this paper, we designed several ablation experiments. The results showed that the pruning strategies can make a great contribution to the reduction of the search space, which can improve the efficiency of the algorithm. Besides, to further improve the efficiency, we designed the other algorithm TotalSR$^+$ using an auxiliary antecedent record table (\textit{ART}) to maintain the antecedent of the sequences that cannot generate a legal ToSR. With the help of \textit{ART}, TotalSR$^+$ can significantly reduce execution time. Finally, extensive experiments conducted on different datasets showed that the proposed two algorithms can not only effectively and efficiently discover all HTSRs but also have outstanding scalability.

In the future, we are looking forward to developing several algorithms based on TotalSR$^+$ that can discover target rules \cite{gan2022towards} or process more complicated data like interval-based events \cite{lee2009mining}. All these fields would be interesting to be exploited.

\section*{Acknowledgment}

This research was supported in part by the National Natural Science Foundation of China (Grant Nos. 62002136 and 62272196), Natural Science Foundation of Guangdong Province of China (Grant No. 2020A1515010970), Shenzhen Research Council (Grant No. GJHZ20180928155209705), and Guangzhou Basic and Applied Basic Research Foundation (Grant No. 202102020277).

\bibliographystyle{ACM-Reference-Format}
\bibliography{TotalSR.bib}
\end{document}